%% file: ms.tex
\documentclass{article}
\pdfoutput=1
\usepackage[ruled,linesnumbered]{algorithm2e}

\usepackage{enumitem}
\usepackage{subcaption}
\usepackage{multirow}
\usepackage{pgfplots}
\usepackage{pgfplotstable}
\usepackage{macros}
\input{lc_macros}

\usepackage[numbers]{natbib}
\pgfplotsset{compat=1.18}

\usepackage{geometry}
\date{}
\usepackage{color}

\title{The Single Robot Line Coverage Problem:\\ Theory, Algorithms, and Experiments}

\author{Saurav Agarwal\thanks{GRASP Laboratory, University of Pennsylvania, USA. The presented work was done at University of North Carolina at Charlotte, USA. E-mail: {\tt SauravAg@upenn.edu}}\; and Srinivas Akella%
\thanks{Department of Computer Science, University of North Carolina at Charlotte, USA. E-mail: {\tt sakella@charlotte.edu}\\ This work was supported in part by NSF award IIP-1919233, a UNC IPG award, and DARPA HR00111820055.}}
\begin{document}
\maketitle

\begin{abstract}
	\input{abstract}
\end{abstract}

\section{Introduction}
\label{sc:introduction}
\input{introduction}

\section{Related Work}
\label{sc:related_work}
\input{related_work}

\section{The Single Robot Line Coverage Problem}
\label{sc:problemDef}
\input{problem_statement}

\section{Approximation Algorithms}
\label{sc:approx_algos}
\input{approx_algo}

\section{Simulations and Experiments}
\label{sc:sim}
\input{simulations}

\section{Conclusion}
\label{sc:conclusion}
\input{conclusion}

\section*{Acknowledgments}
The map tiles in the figures use map data from Mapbox and OpenStreetMap and their data sources: \url{https://www.mapbox.com/about/maps/} and \url{http://www.openstreetmap.org/copyright}.
We thank Gurobi for making their ILP solvers available for academic use.
We also thank Keld Helsgaun for releasing the programs LKH and GLKH for solving the ATSP and the GTSP.

\bibliographystyle{IEEEtranN}
\input{ms.bbl}


\end{document}

%% file: lc_macros.tex
\newcommand*{\scost}[1]{c_s(#1)\xspace} 
\newcommand*{\dcost}[1]{c_d(#1)\xspace} 
\newcommand*{\rcost}[1]{c_r(#1)\xspace} 
\newcommand*{\fcost}[1]{c_f(#1)\xspace} 

\newcommand*{\flow}[1]{f_{#1}} 
\newcommand*{\sv}[1]{s_{#1}} 
\newcommand*{\dd}[1]{d_{#1}} 
\newcommand*{\rv}[1]{r_{#1}} 

\DeclareMathOperator{\indeg}{indeg}
\DeclareMathOperator{\outdeg}{outdeg}

%% file: abstract.tex
Line coverage is the task of servicing a given set of one-dimensional features in an environment.
It is important for the inspection of linear infrastructure such as road networks, power lines, and oil and gas pipelines.
This paper addresses the \textit{single robot line coverage} problem for aerial and ground robots by modeling it as an optimization problem on a graph.
The problem belongs to the broad class of arc routing problems and is closely related to the rural postman problem (RPP) on asymmetric graphs.
The paper presents an integer linear programming formulation with proofs of correctness.
Using the minimum cost flow problem, we develop approximation algorithms with guarantees on the solution quality.
These guarantees also improve the existing results for the asymmetric RPP.
The main algorithm partitions the problem into three cases based on the structure of the \textit{required graph}, i.e., the graph induced by the features that require servicing.
We evaluate our algorithms on road networks from the 50 most populous cities in the world, consisting of up to 730 road segments.
The algorithms, augmented with improvement heuristics, run within 3\,s and generate solutions that are within 10\% of the optimum.
We experimentally demonstrate our algorithms with commercial UAVs.

%% file: introduction.tex
\textit{Line coverage} is the task of servicing linear environment features using sensors or tools mounted on a robot.
The features to be serviced are modeled as one-dimensional segments (or curves); all points along the segments must be visited.
Consider a natural disaster scenario such as flooding in which an uncrewed aerial vehicle (UAV) with cameras is deployed for the assessment of connectivity of a road network for emergency services.
The UAV must traverse the line segments corresponding to the road network and use its cameras to capture images.
It may also travel at higher speeds from one point to another while not capturing images.
The following question then arises: How should a tour for the robot be planned such that it traverses each road network segment and minimizes the flight time?
\fgref{fig:uncc_example} depicts such a scenario with an optimal tour for a UAV and an orthomosaic generated from the images collected during the flight.
Power lines and oil and gas pipelines have similar linear infrastructure that require frequent inspection.
Additional applications arise in perimeter inspection and surveillance, traffic analysis of road networks, and welding operations.
Line coverage algorithms can also be used as a subroutine for routing in area coverage problems, i.e.,~coverage of 2D regions, by decomposing the environment into line segments.

Line coverage is closely related to arc routing problems (ARPs)~\citep{CorberanL14}.
ARPs have been used for snow plowing, spraying salt, and cleaning road networks \citep{CorberanEHPS21}.
ARPs and their algorithms have been traditionally designed for human-operated vehicles.
The above tasks can potentially be automated with uncrewed ground vehicles (UGVs).
Recently, ARP variants, such as the drone arc routing problem, have been developed specifically for UAVs.
However, line coverage has received limited attention in the robotics community.
Developing algorithms that rapidly generate high-quality solutions is essential for planning robot motions.
Algorithms with low computation requirements can be deployed conveniently on robots and executed to replan routes if the environment changes.
In this paper, we design algorithms for line coverage using autonomous systems such as UAVs and UGVs.
The algorithms provide theoretical guarantees on the quality of the solutions, which improve existing results for related single vehicle ARPs.
The simulations and experiments further demonstrate the effectiveness of our algorithms for line coverage applications using robots.

The line coverage problem, modeled using a graph, has two defining attributes:
(1)~The edges in the graph are classified as required and non-required, and
(2)~Robots have two modes of travel---servicing and deadheading.
{\em Required edges} correspond to the linear features to be covered, and the {\em non-required edges} can be used by a robot to travel from one vertex to another to reduce cost.
The vertices in the graph represent the endpoints of the edges.
The robot is said to be {\em servicing} a required edge when it performs task-specific actions such as collecting sensor data.
Each required edge needs to be serviced exactly once.
The robot may also traverse an edge without performing servicing tasks to optimize the travel time, conserve energy, or reduce the amount of sensor data.
This is known as {\em deadheading}, and both types of edges may be used any number of times for this purpose.

\begin{figure}[htbp]
	\centering
	\begin{subfigure}[c]{.3\textwidth}
		\centering
		\includegraphics[width=\textwidth]{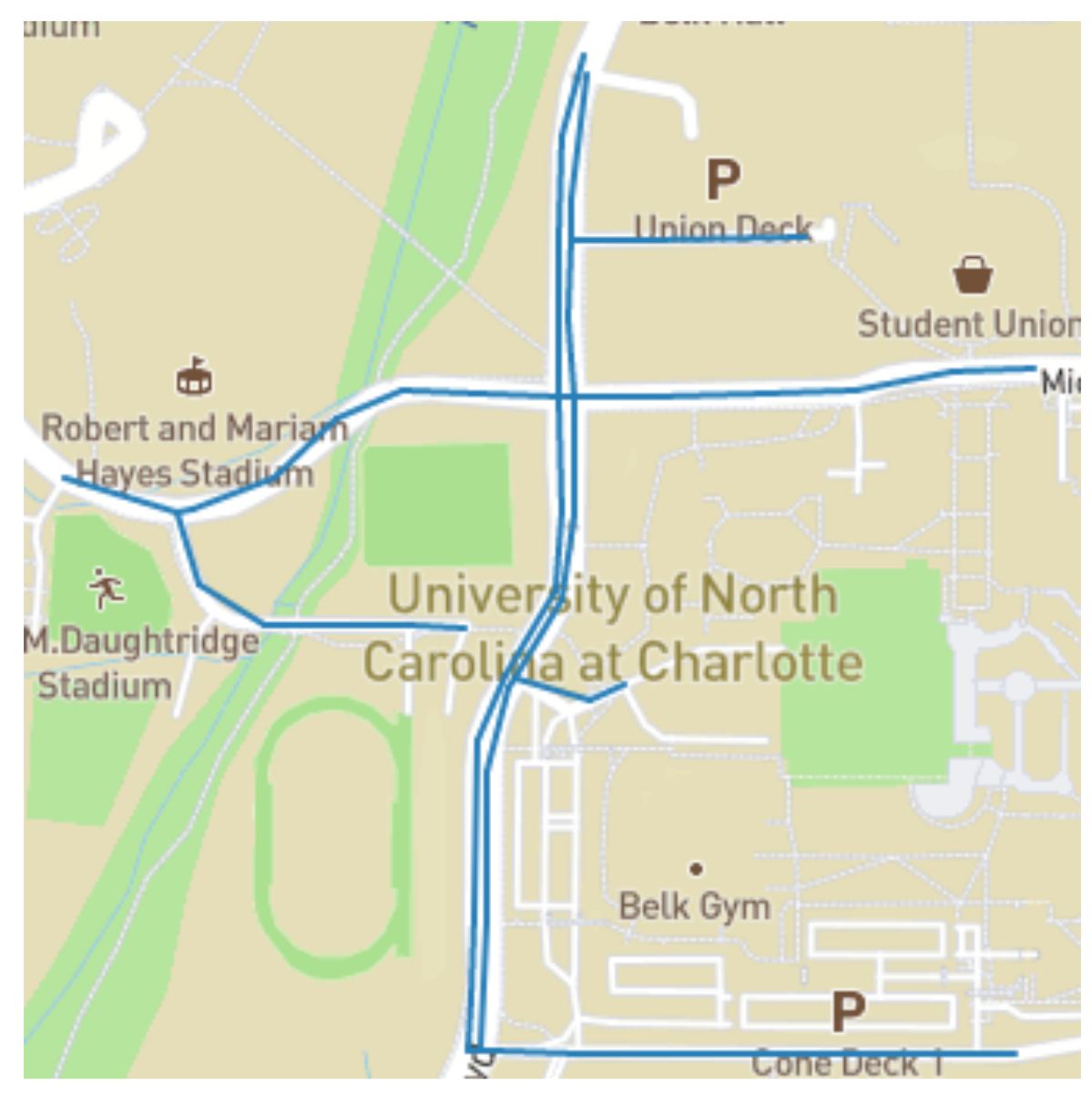}
		\caption{Road network}
	\end{subfigure}%
	\hfill
	\begin{subfigure}[c]{.3\textwidth}
		\centering
		\includegraphics[width=\textwidth]{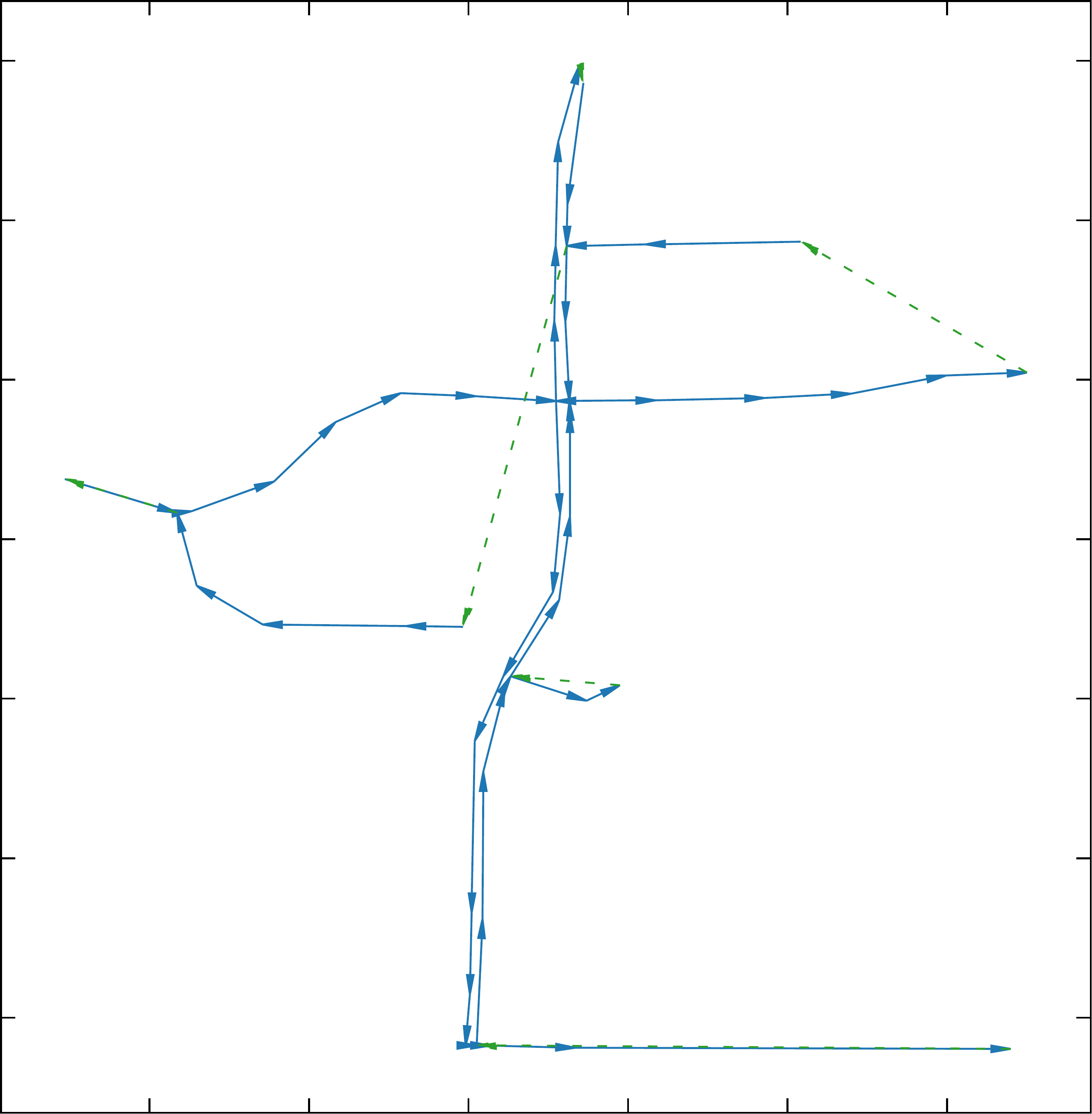}
		\caption{Line coverage tour}
	\end{subfigure}%
	\hfill
	\begin{subfigure}[c]{.3\textwidth}
		\centering
		\includegraphics[width=\textwidth]{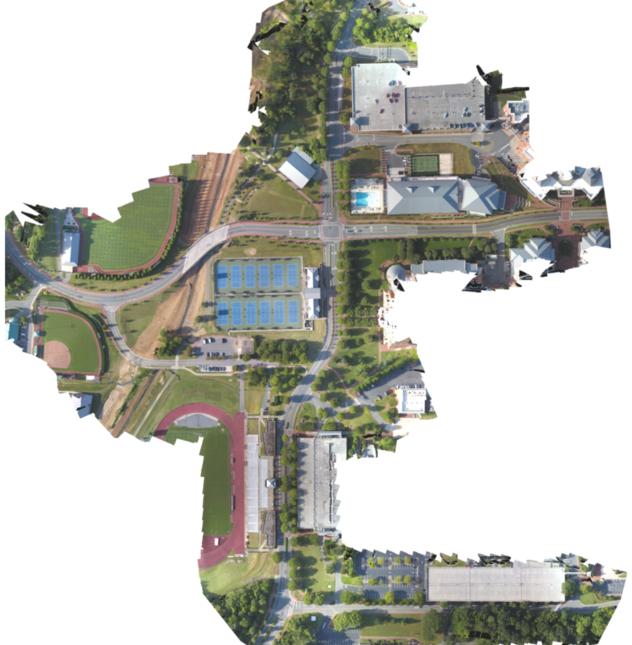}
		\caption{Orthomosaic}
	\end{subfigure}
	\caption{%
		Line coverage of a road network by an autonomous UAV:
		(a)~A region of the UNC Charlotte campus road network; blue lines show required edges to be serviced.
		Non-required edges, not shown, are straight lines between pairs of vertices.
		The network data was extracted using OpenStreetMap.
		(b)~An optimal coverage tour for the road network is shown; dashed segments indicate deadheading travel.
		(c)~An orthomosaic map of the road network generated from photos taken by the UAV along the required edges of the coverage tour.%
	\label{fig:uncc_example}}
\end{figure}

A service cost and a deadhead cost (e.g., travel time) are associated with each required edge, and they are incurred each time an edge is serviced or deadheaded, respectively.
Only the deadhead cost is associated with the non-required edges.
The sum of the service and deadhead costs is to be minimized.
As task-related actions are performed only when servicing a required edge, not while deadheading, the service costs are considered to be greater than or equal to the deadhead costs.
For example, with travel time as the cost, a UAV servicing an edge by recording images may need to travel slower than when deadheading to avoid motion blur.
In contrast, the service and deadhead costs are usually assumed to be identical for the required edges in standard ARPs, and therefore, the line coverage problem generalizes the standard problems.

In many robotics applications, the cost of travel is direction dependent.
For example, for ground robots, the cost of traveling uphill can be significantly higher than that of traveling downhill.
Similarly, for UAVs, the cost of an edge may differ along the two directions due to wind conditions.
Hence, we consider the graph to have asymmetric edge costs for both servicing and deadheading.
Asymmetric edges can also model one-way streets for ground robots.

The {\em single robot line coverage problem} is the problem of finding a coverage tour that minimizes the total travel cost while ensuring that each linear feature is serviced exactly once.
The formulation allows distinct costs for servicing and deadheading and permits asymmetric costs.
The single robot line coverage problem is a generalization of the rural postman problem (RPP) introduced by Orloff~\cite{Orloff74}.
The NP-hardness of the RPP, shown by Lenstra and Kan~\cite{LenstraK76}, implies that the single robot line coverage problem is NP-hard.
This makes it imperative to develop approximation algorithms.
Agarwal and Akella~\cite{AgarwalA20ICRA} addressed the line coverage problem with multiple robots and presented two heuristic algorithms.
These heuristic algorithms do not have theoretical guarantees, though they were empirically shown to give results close to optimal.
The single robot line coverage problem, as discussed in this paper, is a special case of the multiple robot version that does not consider the robot resource capacity and the resource demands of edges.
Algorithms for the single robot problem can be used to develop algorithms for multiple robots.

{\em Contributions:}
In this paper, we elucidate the single robot line coverage problem and develop approximation algorithms.
We analyze the problem in stages---going from a simpler problem to the most general version.
The problems are based on the structure of the {\em required graph}, i.e., the graph induced by only the required edges.
The contributions of the paper are:
\begin{enumerate}
	\item We pose the single robot line coverage problem as an optimization problem and develop an integer linear programming (ILP) formulation that gives an optimal solution.
		Additionally, we provide formal proofs for the correctness of the formulation.
	\item We develop a relaxation of the ILP formulation with continuous decision variables and model it using a minimum cost flow problem.
		The model is used to design an optimal algorithm for graphs with Eulerian required graphs.
		A 2-approximation algorithm is then developed for connected required graphs.
	\item An $(\alpha(C)+2)$-approximation algorithm is presented for the general case of a required graph with multiple connected components, where $C$ is the number of connected components, and $\alpha(C)$ is the approximation factor for an algorithm for the asymmetric traveling salesperson problem.
		Proofs for the approximation guarantee are provided for all the algorithms.
	\item	We publish a dataset\footnote{The dataset is available at: \url{https://github.com/UNCCharlotte-CS-Robotics/LineCoverage-dataset}} consisting of road networks of the 50 most populous cities in the world and perform simulations.
		The results show that the algorithms compute solutions within 10\% of the optimum.
		The algorithms find solutions to the instances within 3\,s and are fast enough to enable rapid robot replanning.
		Experimental validation of the algorithms is performed.
		We also provide an open-source implementation\footnote{The source code is available at: \url{https://github.com/UNCCharlotte-CS-Robotics/LineCoverage-library}} of our algorithms.
\end{enumerate}

Building on the earlier publication by the authors~\cite{AgarwalA20WAFR}, this paper develops a thorough theoretical analysis of the formulations and the algorithms, provides extensive simulation results, and validates the algorithms in experiments with UAVs.
In particular, formal proofs for the correctness of the ILP formulation are provided.
Detailed theoretical analysis of the problem, with a running example, furnishes insights into the structure of the problem that has been instrumental in developing the algorithms.
Further improvements to the algorithms are made using heuristic subroutines.
These lead to an efficient and fast implementation that we demonstrate on a new dataset of road networks.
We additionally demonstrate the application of our algorithms through two experiments on a campus road network.

{\em Organization:} The rest of the paper is organized as follows.
The related work is discussed in \scref{sc:related_work}.
The single robot line coverage problem is formally described in \scref{sc:problemDef}.
The section provides the ILP formulation and a continuous relaxation along with formal proofs.
The approximation algorithms are developed in \scref{sc:approx_algos}.
The simulations and experiments are discussed in \scref{sc:sim}.
\scref{sc:conclusion} concludes the paper.

%% file: related_work.tex
The line coverage problem belongs to the broad class of arc routing problems (ARPs).
A hierarchy of standard arc routing problems and the single robot line coverage problem is shown in \fgref{fig:hierarchySingle}.
The ARPs have traditionally been applied to transportation problems in which servicing is related to tasks such as delivery and pick up of goods \cite{CorberanL14}.
Hence, the travel distances are used as costs and often have the same value whether the edge is serviced or deadheaded.
Separate and asymmetric service costs are typically not considered.

\begin{figure}[htbp]
	\centering
	\input{graphics/hierarchy_singleVehicle}
	\caption{A hierarchy of arc routing problems with a single vehicle/robot.
		An arrow from problem A to problem B indicates that B is a special case of A.
		The single robot line coverage problem generalizes all the other depicted arc routing problems.
	Postman problems on asymmetric graphs are also termed {\em windy}, for example, the windy postman problem (WPP) and the windy rural postman problem.}
	\label{fig:hierarchySingle}
\end{figure}
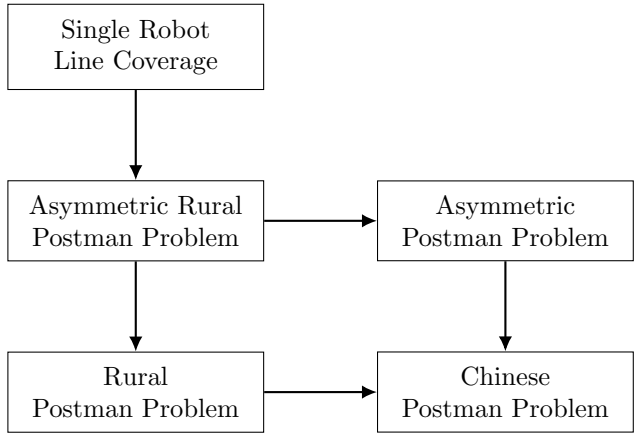

{\bf Arc Routing Problems for a Single Vehicle:}
The Chinese postman problem (CPP) is to find an optimal tour such that every edge in a given undirected and connected graph is traversed at least once.
Edmonds and Johnson~\cite{EdmondsJ73} used matching and network flows to solve the CPP on undirected, directed, and Eulerian mixed graphs.
The CPP on mixed or asymmetric graphs is an NP-hard problem.
Frederickson~\cite{Frederickson79} presented a $5/3$-approximation algorithm for the CPP on mixed graphs by using a combination of two approximation algorithms.
The approximation factor was later improved to $3/2$ by Raghavachari and Veerasamy~\cite{RaghavachariV99}.

The asymmetric postman problem, also known as the windy postman problem (WPP), is the CPP with asymmetric edge costs.
This problem is NP-hard, as shown by Guan~\cite{Guan84}.
Win~\cite{Win89} solved the WPP for Eulerian graphs in polynomial time by modeling it as a minimum cost flow problem.
Win also designed a 2-approximation algorithm for WPP on general graphs using matching (to make the graph Eulerian) and minimum cost network flow.
Raghavachari and Veerasamy~\cite{RaghavachariV99SODA} gave a $3/2$-approximation for the WPP.
The CPP and the WPP do not allow non-required edges in the graph.

When the edges to be serviced are a subset of the edges in the graph, we have the rural postman problem (RPP).
It was proved that the RPP is NP-hard by Lenstra and Kan~\cite{LenstraK76}.
For the RPP, Frederickson~\cite{Frederickson79} gave a $3/2$-approximation algorithm similar to the algorithm by Christofides~\cite{Christofides76} for the metric traveling salesperson problem~(TSP).
The asymmetric RPP considers undirected graphs with asymmetric edge costs, i.e., the cost of traversal of an edge can be different in the two directions.
This problem is more commonly referred to as the windy rural postman problem in the arc routing literature.
Lower bounds and heuristic algorithms were proposed by Benavent et al.~\cite{BenaventCCSV07,BenaventCPPS04} for the windy rural postman problem.
However, these algorithms do not provide an approximation guarantee and have known pathological instances for which the ratio of the solution cost to the optimal cost is infinite.
van Bevern et al.~\cite{vanBevernKS17} showed that if the $n$-vertex asymmetric traveling salesperson problem (ATSP), subject to the triangle inequality constraint, is $\alpha(n)$-approximable in $t(n)$ time, then $n$-vertex RPP on an asymmetric and mixed graph is $(\alpha(C)+3$)-approximable in $O(t(C) + n^3\log n)$ time, where $C$ is the number of weakly connected components in the subgraph induced by required arcs and edges.
The single robot line coverage problem is closely related to the asymmetric RPP.
However, in the asymmetric RPP, the costs of deadheading and servicing a required edge are the same.
In contrast, the formulation in this paper allows distinct and asymmetric servicing and deadheading costs for required edges.
The algorithms presented in this paper are applicable to the asymmetric RPP as any instance of the asymmetric RPP can be converted to that of the single robot line coverage problem by setting the cost of deadheading a required edge to the cost of the edge.
The guarantee on the approximation factor for our algorithms is an improvement over the previously best-known algorithms given by van Bevern et al.~\cite{vanBevernKS17} for the asymmetric RPP.

Exact and metaheuristic methods have been proposed for the ARPs, and they are covered in the survey paper by Corber{\'a}n and Prins~\cite{CorberanP10} and the monograph by Corber{\'a}n and Laporte~\cite{CorberanL14}.
Exact methods include branch-and-cut with specific cutting plane procedures, branch-and-price, and column generation.
One of the key techniques for these algorithms is to incorporate additional constraints that tighten the feasible space of the linear relaxation.
Metaheuristic algorithms, such as scatter search, tabu search, and variable neighborhood descent, have also been used to solve the ARPs.
However, these algorithms are not particularly suitable for robotics applications as they require significant computational resources.
Moreover, they typically require a good initial solution as an additional input to upper bound the optimal cost of the instance.
The algorithms presented in this paper can be used to provide such an initial solution with a guaranteed upper bound provided by the approximation factor.

{\bf Asymmetric Traveling Salesperson Problem (ATSP):}
A dynamic programming algorithm that runs in $\mathcal O(n^22^n)$ computation time, where $n$ is the number of vertices in the input graph, was given by Held and Karp~\cite{HeldK62} and Bellman~\cite{Bellman62}.
The algorithm gives optimal solutions and can effectively solve small instances.
Svensson et al.~\cite{SvenssonTV18ATSP} were the first to present a constant-factor approximation algorithm for the ATSP with the triangle inequality.
Traub and Vygen~\cite{TraubV20ATSP} improved the approximation ratio to $22+\epsilon,$ $\epsilon > 0$, for the ATSP.
These results are relevant for the theoretical guarantees on the approximation factor of our algorithms.

{\bf Line Coverage in Robotics:}
Line coverage has been used in robotics for the inspection of road networks and object boundaries.
Dille and Singh~\cite{DilleS13} model the problem of road network coverage through tessellation of the road network by circles corresponding to the sensor footprint and finding a subset of the circular regions that covers the entire road network.
Algorithms for node routing problems, such as the TSP and multiple TSP, are then used to find tours for the robots.
Oh et al.~\cite{OhKTW14} presented a mixed integer linear programming formulation and a heuristic algorithm for coverage of road networks using Dubins curves with Euclidean distances as costs.
The nearest insertion heuristic method, originally designed for the TSP, is used to find a sequence of edges to be visited while incorporating Dubins curves.
The sequence is then split across a robot team using an auction algorithm.
Easton and Burdick~\cite{EastonB05} proposed a constructive heuristic algorithm for the RPP with $k$ vehicles for coverage of 2D object boundaries.
The algorithm first groups the required edges into $k$ clusters and computes a representative edge for each cluster.
Additional edges are added to each cluster to ensure connectivity.
Tours are computed for each cluster independently using the polynomial-time CPP algorithm proposed by Edmonds and Johnson~\cite{EdmondsJ73}.
Williams and Burdick~\cite{WilliamsB06} developed algorithms for boundary inspection while considering revisions to the path plan for the robots to account for the changes in the environment and in the robot team sizes.
Xu and Stentz~\cite{XuS10} use CPP and RPP formulations for line coverage and consider the case where the prior map information may be incomplete.
They propose heuristic algorithms that can regenerate solutions rapidly when new map information is incorporated with the prior map.
They extended this work to multiple robots \cite{XuS11ICRA} using $k$-means clustering to decompose networks into smaller components, similar to the algorithm presented by Easton and Burdick~\cite{EastonB05}.
Campbell et al.~\cite{CampbellCPS18} presented an application of ARPs to generate route for a single UAV, where they allow the UAV to service a required edge in parts.
The UAV may service a part of a required edge, move to some other edge, and come back later to service the remaining parts of the required edge.
They convert the problem into standard ARPs by discretizing each required edge.
The costs are considered to be Euclidean distances.
Our algorithms are directly applicable to this discretized version of the problem.

These papers illustrate various applications of the line coverage problem.
However, they do not consider asymmetric edge costs or distinct service and deadhead costs.
Moreover, the heuristic algorithms do not provide theoretical guarantees on the quality of the solutions.
The formulation and the algorithms presented in this paper address these shortcomings of the prior work.

{\bf Arc Routing Problems in Area Coverage:}
Arc routing problems have been used in robotics primarily as a subroutine in area coverage problems to generate efficient routes for a robot.
Arkin et al.~\cite{ArkinFM00} use an algorithm similar to the one given by Edmonds and Johnson~\cite{EdmondsJ73} for the CPP to find routes for a robot for the {\em milling problem}, a variant of the area coverage problem wherein the tool is constrained within the workspace.
Mannadiar and Rekleitis~\cite{MannadiarR10} formulate the area to be covered in terms of edges in a Reeb graph.
Optimal solutions to the CPP were used to compute an Euler tour for coverage of available free space while minimizing the path length.
Karapetyan et al.~\cite{KarapetyanBMTR17} use the CPP formulation for $k$ robots to find routes for multiple robots on a Reeb graph.
They used the CPP to compute a large Euler tour and then break it into smaller tours using the algorithm given by Frederickson et al.~\cite{FredericksonHK76}.
Our algorithms for the single robot line coverage problem are directly applicable to the above techniques to generate routes for the area coverage problem.

%% file: graphics/hierarchy_singleVehicle.tex
\begin{tikzpicture}
	\node[myBlock,fill=none] (lcs) at (-5,0) [text width=3cm, align=center] {Single Robot Line Coverage};
	\node[myBlock, fill=none] (wrpp)[below=1.2cm of lcs] {Asymmetric Rural Postman Problem};
	\node[myBlock, fill=none] (wpp) [right=1.5cm of wrpp] {Asymmetric\\Postman Problem};
	\node[myBlock, fill=none] (cpp) [below=1.2cm of wpp] {Chinese\\ Postman Problem};
	\node[myBlock, fill=none] (rpp) [below=1.2cm of wrpp] {Rural\\Postman Problem};
	\draw[Arc] (lcs) -- (wrpp);
	\draw[Arc] (wrpp) -- (wpp);
	\draw[Arc] (wrpp) -- (rpp);
	\draw[Arc] (wpp) -- (cpp);
	\draw[Arc] (rpp) -- (cpp);
\end{tikzpicture}

%% file: problem_statement.tex
We now model the single robot line coverage problem as an optimization problem on a graph.
We are given a connected undirected graph $G=(V, E, E_r)$, where $V$ is the set of vertices, $E$ is the set of edges, and $E_r\subseteq E$ is the set of required edges.
The set $E$ can contain parallel edges between two vertices, i.e., we allow $G$ to be a multigraph.
The service and deadhead costs are given as inputs along with the graph.
The {\em single robot line coverage problem} is to find a coverage tour that minimizes the total cost of travel on the graph, such that each of the required edges in $E_r$ is serviced exactly once.

For each edge $e$ in $E$, we associate two directional arcs $a_e$ and $\bar a_e$ that are opposite in direction to one another.
If a robot {\em services} a required edge $e \in E_r$ in the direction $a_e$, then a service cost $\scost{a_e}$ is incurred; similarly for the direction $\bar a_e$.
If a robot traverses an edge without servicing it, the robot is said to be {\em deadheading}; for example, this occurs when a robot is traveling from a vertex of an edge to that of another edge using a non-required edge.
Both required and non-required edges may be deadheaded.
Deadhead costs for an edge $e$ are denoted by $\dcost{a_e}$ and $\dcost{\bar a_e}$.
We use $\scost{A}$ and $\dcost{A}$ to denote the corresponding sums of the service and deadhead costs for a set of arcs $A$.
We denote by $\bar A$ the set of arcs oppositely directed to the arcs in $A$.

We consider the edge costs, for both servicing and deadheading, to be direction dependent, i.e., the graph is asymmetric.
For example, $\scost{a_e}$  may differ from $\scost{\bar a_e}$.
The service and deadhead costs can be arbitrary positive numbers, with the constraint that the service cost of an edge is no less than the deadhead cost in the same direction.
The costs, such as travel time, appear in the objective function of the problem.
Since we allow edge costs to be asymmetric, the model allows both directed and mixed graphs.
This can be achieved by setting the cost of the arcs that the robot is not allowed to travel to infinity or a very large constant.
Additionally, we need to ensure that the graph is strongly connected so that there is a feasible solution.
We also allow multiple copies of the edges and can model repeated servicing of segments.

\subsection{Preliminaries}
\label{sec:prelims}
Let $G=(V, E, E_r)$ be a connected undirected graph for the line coverage problem, such that $E_r\subseteq E$.
The subgraph $G_r=(V_r, E_r)$ induced by the set of required edges $E_r$ is called the {\em required graph} of $G$; $V_r\subseteq V$ is the set of vertices that have at least one edge in $E_r$ incident on them.
The set of non-required edges is denoted by $E_n=E\setminus E_r$.
We define the set of all arcs to be $\mathcal A = \bigcup \{a_e, \bar a_e\},\; \forall e\in E$.
Similarly, $\mathcal A_r$ is defined for the set of required edges.
If an arc~$a$ represents the travel direction from vertex $u$ to vertex $v$, then the vertices $u$ and $v$ are called the tail $t(a)$ and  head $h(a)$ of $a$, respectively.
We denote by $H(\mathcal A, v)$ all the arcs $a\in \mathcal A$ that have $v$ as the head.
Similarly, $T(\mathcal A, v)$ is defined for the tail.
The {\em degree} of a vertex $v\in V$ is the number of edges incident on $v$.
A {\em walk} in a graph $G$ is a non-empty alternating sequence $v_1e_1v_2e_2\ldots e_{k}v_{k+1}$ of vertices in $V$ and edges in $E$ such that $v_i$ and $v_{i+1}$ are the end vertices of an edge $e_i$ for all $1 \leq i \leq k$.
A {\em closed walk} is a walk with the same start and end vertex, i.e., $v_1=v_{k+1}$.
An {\em Euler tour} is a closed walk such that every edge in the graph is traversed exactly once.
A graph that has an Euler tour is called {\em Eulerian}.
It is well established that an undirected graph is Eulerian if and only if every vertex has an even degree, see, e.g., \cite{PapadimitriouSBook}.

Let $D=(V, A)$ be a {\em directed graph} (digraph) with $V$ as the set of vertices and $A$ as the set of (directed) arcs.
The digraph is strongly connected if there exists a directed path from any vertex in~$V$ to any other vertex in~$V$.
Analogous to the undirected graph, $T(A, v)$ and $H(A, v)$ are defined for the arc set $A$ and a vertex $v\in V$.
The {\em indegree} of a vertex $v\in V$, denoted by $\indeg(v)$, is the number of arcs entering the vertex~$v$.
Similarly, the {\em outdegree} of a vertex $v \in V$, denoted by $\outdeg(v)$, is the number of arcs going out of the vertex~$v$.
A digraph is Eulerian if and only if the graph is strongly connected and {\em balanced}, i.e., $\indeg(v) = \outdeg(v), \forall v \in V$.
{\em Imbalance} $\mc I(A, v)$ for the arc set $A$ at a vertex $v$ is given by $\outdeg(v) - \indeg(v) = |T(A, v)| - |H(A, v)|$.
Analogous to the undirected graph, a {\em diwalk} is a sequence $v_1a_1v_2a_2\ldots a_{k}v_{k+1}$ of vertices and arcs in a digraph $D=(V, A)$ such that the tail of $a_i$ is $v_i$ and head of $a_i$ is $v_{i+1}$.
A {\em closed diwalk} is a walk with the same start and end vertices.
An {\em Eulerian tour} on an Eulerian digraph is a closed diwalk such that each arc is traversed exactly once.
An Euler tour can be constructed from an Eulerian graph (or digraph) in $\mathcal O(|A|)$ computational time, see, e.g., \cite{PapadimitriouSBook}.

\begin{definition}
	Coverage Tour:\\
	Given a connected graph $G=(V, E, E_r)$, a {\em coverage tour} is a closed walk in the graph~$G$ such that each required edge $e\in E_r$ is serviced exactly once.
\end{definition}
Note that in a coverage  tour, a required or a non-required edge may be used multiple times for deadheading.
We define the following variables:
\begin{equation}
	\begin{aligned}
	&\sv{a_e}, \sv{\bar a_e} \in \{0, 1\}, \text{ and } \sv{a_e} + \sv{\bar a_e} = 1\ \quad\forall e \in E_r\\
	&\dd{a_e}, \dd{\bar a_e} \in \mathbb N\cup \{0\}\ \quad\forall e\in E
	\end{aligned}
\end{equation}
The variables $\sv{a_e}$ and $\sv{\bar a_e}$ represent the two opposite directions of servicing the edge $e$; exactly one of the two can be equal to $1$ for a valid coverage tour.
The variables $\dd{a_e}$ and $\dd{\bar a_e}$ represent the number of times an edge is deadheaded in the corresponding direction.
The cost of a coverage tour~$\tau$ is to be minimized and is given by:
\begin{equation}
	c(\tau) = \sum_{e\in E_r}\big[\scost{a_e} \sv{a_e}+ \scost{\bar a_e}\sv{\bar a_e}\big] 
	+ \sum_{e\in E}\big[\dcost{a_e} \dd{a_e}+ \dcost{\bar a_e}\dd{\bar a_e}\big]  \label{eqn:costCoverageTour}
\end{equation}
For a valid coverage tour $\tau$, we can create an Eulerian digraph $D_{\tau}=(V, A_{\tau})$ from these variables by adding each arc as many times as the value of its corresponding variable.
The digraph will have the same total cost as the coverage tour, i.e., $c(\tau) = c(A_{\tau})$.
A closed diwalk can be obtained from $D_{\tau}$ in $\mathcal O(|A_{\tau}|)$ time.
We will often use this equivalent Eulerian digraph representation of a coverage tour in the rest of the paper.

The following subsection presents an integer linear programming (ILP) formulation for the single robot line coverage problem.
The ILP formulation allows us to formally define the problem in the form of an objective and a set of constraints, and provides an optimal solution.
In subsequent sections, we provide a relaxation of the ILP formulation with continuous decision variables and relate the relaxation to a network flow model.
The network flow model forms the basis of our approximation algorithms in Section~\ref{sc:approx_algos}.

\subsection{Integer Linear Programming (ILP) Formulation}
The standard ILP formulations for arc routing problems involve an exponential number of constraints for ensuring connectivity \citep{CorberanL14} and are usually difficult to incorporate into standard ILP solvers.
Therefore, we adopt the formulation presented by Gouveia et al.~\cite{GouveiaMP10} as later specialized by Agarwal and Akella~\cite{AgarwalA20ICRA}.

The primary decision variables of the formulation comprise the service variables $\sv{a_e}$ and $\sv{\bar a_e}$ and the deadheading variables $\dd{a_e}$ and $\dd{\bar a_e}$, as described in \scref{sec:prelims}.
The formulation needs an additional set of variables $z_a,\ \forall a\in \mc A$.
These variables help in ensuring that the solution digraph computed by the ILP formulation is strongly connected, and therefore provides the corresponding coverage tour.
The variables~$z_a$ can be interpreted as generalized flows across the arcs and are illustrated further in Theorem~\ref{thm:ilp}.
Additionally, we select an arbitrary vertex $v_0\in V_r$, i.e., a vertex corresponding to one of the ends of a required edge and will be traversed in a coverage tour.
This vertex acts as a source node for the generalized flow and is critical for proving the connectivity of the solution digraph.\\

\noindent {\bf SRLC-ILP} (single robot line coverage ILP formulation)\\
\noindent Minimize:\\
\begin{equation}
	c(\tau) = \sum_{e\in E_r}\big[\scost{a_e} \sv{a_e}+ \scost{\bar a_e}\sv{\bar a_e}\big] 
	+ \sum_{e\in E}\big[\dcost{a_e} \dd{a_e}+ \dcost{\bar a_e}\dd{\bar a_e}\big]  \label{eqn:srlc:obj}
\end{equation}
subject to:
\begin{align}
	& \sumSr{a_1 \in H(\mc A_r, v)}\sv{a_1}
	\,+
	\sumSr{b_1\in H(\mc A, v)}\dd{b_1}
	\, -
	\sumSr{a_2 \in T(\mc A_r, v)}\sv{a_2}
	\,-
	\sumSr{b_2\in T(\mc A, v)} \dd{b_2} =0,\quad\forall v \in V\label{eqn:srlc:symmetry}\\
	&\sv{a_e} + \sv{\bar a_e} = 1, \quad\forall e \in E_r\label{eqn:svlc:service}\\
	&\sumSr{a\in H(\mc A, v)}z_a \,- \sumSr{a\in T(\mc A, v)} z_a \,= \sumSr{a\in H(\mc A_r, v)}s_a, \quad \forall v \in V\setminus \{v_0\}\label{eqn:srlc:edgeFlow}\\
	&\sumSr{a\in T(\mc A, v_0)} z_a \,= \sumS{a\in \mc A_r}s_a \,=| E_r|\label{eqn:srlc:flowDepot}\\
	&z_a \leq |E_r|(d_a+s_a), \quad \forall a\in \mc A_r\label{eqn:srlc:flowLimit1}\\
	&z_a \leq |E_r|d_a, \quad \forall a\in \mc A\setminus \mc A_r\label{eqn:srlc:flowLimit}\\
	&\sv{a_e}, \sv{\bar a_e} \in \{0, 1\},\quad \forall e \in E_r\label{eqn:srlc:trivialEr}\\
	& \dd{a_e}, \dd{\bar a_e}  \in \mathbb N\cup \{0\},\quad \forall e\in E\label{eqn:srlc:trivialE}\\
	& z_{a_e}, z_{\bar a_e}  \in \mathbb N\cup \{0\},\quad \forall e\in E\label{eqn:srlc:trivialF}
\end{align}

The objective function~\eqref{eqn:srlc:obj} minimizes the cost of a coverage tour.
The  {\em balance} (or symmetry) constraints~\eqref{eqn:srlc:symmetry} state that for
each vertex, the number of arc traversals into a vertex should be equal to the
number of arc traversals out of the vertex.
The {\em servicing} constraints \eqref{eqn:svlc:service} ensure that each required edge is serviced exactly once
and in only one direction.
Constraints~\eqref{eqn:srlc:edgeFlow}--\eqref{eqn:srlc:flowLimit} are {\em connectivity} constraints and are a type of generalized flow constraints.
The vertex~$v_0$ is a source to a flow equal to the number of required edges $|E_r|$ as given by constraints~\eqref{eqn:srlc:flowDepot}.
Constraints~\eqref{eqn:srlc:edgeFlow} state that a flow of one unit is absorbed each time a required edge is serviced.
Building on the analysis by Gouveia et al.~\cite{GouveiaMP10}, we show that these constraints ensure that the solution digraph is connected in Lemma~\ref{lem:connectivity}.
The balance and the connectivity constraints together ensure that the resulting solution digraph has an Euler tour.
The {\em integrality} constraints \eqref{eqn:srlc:trivialEr}--\eqref{eqn:srlc:trivialF} ensure that the decision variables are integers.
In the case of a mixed or a directed graph, where edges are not allowed to be traversed in one of the directions, we have an infinite\footnote{The \href{doi.org/10.1109/IEEESTD.2019.8766229}{IEEE Standard for Floating-Point Arithmetic (IEEE 754)} supports infinite values.
General MILP solvers also provide the functionality to specify infinity.} or a very large constant as the cost.
If the solution digraph generated by the ILP formulation has such a high cost, the solution is treated as infeasible, and there is no corresponding coverage tour.
An input graph with its optimal tour is shown in \fgref{fig:input}.

\begin{figure}[htbp]
	\centering
	\input{./graphics/ilp_sol}
	\caption[Example optimal coverage tour]{In the input graph (a), the solid blue lines and the dashed green lines represent required and non-required edges, respectively.
		The service costs in the two directions are shown in the input graph---the costs are closer to the head of the corresponding arc.
		Deadhead costs for the required edges are half the service costs in the respective directions.
		Non-required edges are set to have a unit cost in both directions.
		Deadheading costs are not shown.
		In the optimal coverage tour (b), the servicing arcs are shown as solid blue arcs, while the deadheadings are shown as green dashed arcs.
		The numbers in Figure (b) indicate the costs of the arcs in the final solution.
	The cost of the coverage tour is 42.\label{fig:input}}
\end{figure}
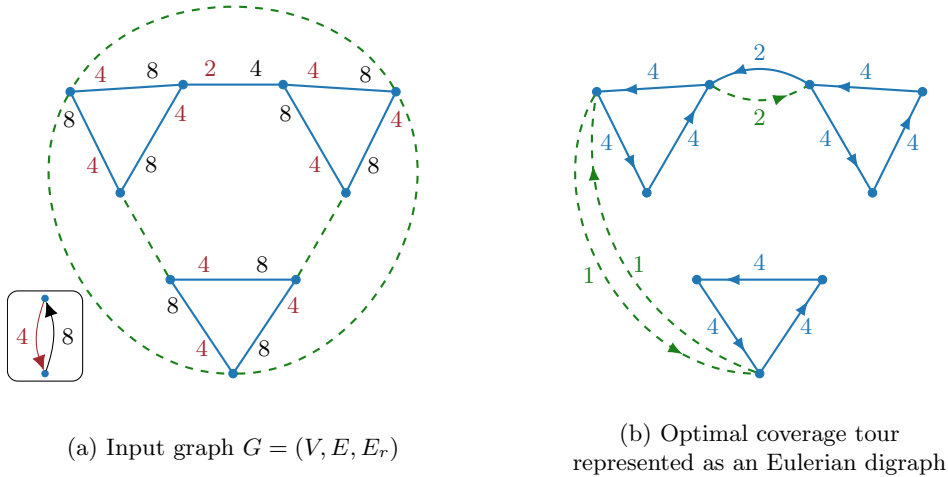

We now prove the correctness of the ILP formulation by showing that the formulation gives an optimal coverage tour (Theorem~\ref{thm:ilp}).
There are two components to the proof:
(1)~The solutions obtained from the SRLC-ILP formulation are Eulerian digraphs that correspond to feasible coverage tours (Lemma~\ref{lem:connectivity}), and
(2)~Any feasible coverage tour has an equivalent feasible solution to the ILP (part of Theorem~\ref{thm:ilp}).
Using the above two statements and the fact that the ILP formulation is an optimization problem with the cost of the coverage tour as the objective function, it follows that the optimal solution to the formulation has an equivalent optimal coverage tour.


\begin{lemma}
	\label{lem:connectivity}
	Given a connected graph $G=(V, E,E_r)$, a solution to the SRLC-ILP formulation has an equivalent Eulerian digraph.
	\label{lem:connected}
\end{lemma}
\begin{proof}
	A digraph is Eulerian if (a)~each vertex in the digraph is balanced, and (b)~the digraph is strongly connected.
	Given a solution to the SRLC-ILP formulation, we create a digraph $D_{\epsilon} = (V, A_{\epsilon})$ with the same vertex set as the input graph and an arc set $A_{\epsilon}$ with $d_a + s_a$ copies of each arc $a$ in $\mc A_r$ and $d_a$ copies for each arc in $\mc A\setminus \mc A_r$.
	Note that the indegree equals the outdegree for each vertex in digraph $D_{\epsilon}$ because of balance constraints~\eqref{eqn:srlc:symmetry}.

	It remains to show that the digraph $D_{\epsilon}$ is strongly connected.
	In particular, we will show that any arc $a\in A_{\epsilon}$ with $s_a > 0$ and/or $d_a>0$ is connected to an arbitrary vertex $v_0\in V_r$, where $V_r$ is the vertex set corresponding to the required graph.
	The vertex~$v_0$ must be traversed by a feasible solution of the SRLC-ILP formulation as the vertex~$v_0$ is connected to at least one of the required edges in~$E_r$.

	For our proof by contradiction, assume that $D_{\epsilon}$ is not connected, i.e., it has more than one connected component.
	Consider one such connected component such that it is not connected to the vertex $v_0$.
	Since the solution digraph is balanced, we can form an Eulerian diwalk that traverses {\em all} arcs in the selected connected component.
	We will assume that this Eulerian diwalk contains at least one arc that is being serviced, for otherwise, it is a diwalk of only deadheading arcs and can be eliminated without an increase in cost.
	Let $S\subset V$ be the set of vertices corresponding to this Eulerian diwalk such that $v_0\notin S$.
	Define $\bar S = V\setminus S$.
	Note that $v_0\in \bar S$.

	Summing the constraints~\eqref{eqn:srlc:edgeFlow} over all the vertices in $S$ gives the following equation:
	\begin{equation}
		\sumSr{v\in S}\left(\sumSr{a\in H(\mc A, v)}z_a \,- \sumSr{a\in T(\mc A, v)} z_a\right) \,= \sumSr{v\in S}\left(\sumSr{a\in H(\mc A_r, v)}s_a\right)\label{eqn:lemmaConnected:sum}
	\end{equation}
	For the purposes of this proof, define the following for any pair of sets $F, G \subseteq V$:
	\begin{equation}
		\begin{split}
			\delta(F, G)&= \{a\in A_{\epsilon} \mid t(a) \in F, h(a)\in G\}\\
			R(F, G) &= \sumSr{a\in \delta(F, G)}s_a, \quad N(F, G) = \sumSr{a\in \delta(F, G)}d_a\\
			Z(F, G) &= \sumSr{a\in \delta(F, G)}z_a\\
		\end{split}
	\end{equation}
	Then \eqref{eqn:lemmaConnected:sum} can be written as:
	\begin{equation}
		Z(\bar S, S) - Z(S, \bar S) = R(S, S) + R(\bar S, S)
	\end{equation}
	Note that $R(S, S) > 0$ and $R(\bar S, S) = 0$ from our assumption for contradiction.
	The term $Z(S, \bar S)$ is non-negative.
	This implies that $Z(\bar S, S)$ is strictly positive.
	Summing the constraints~\eqref{eqn:srlc:flowLimit1} and \eqref{eqn:srlc:flowLimit} over all arcs $a\in \delta(\bar S, S)$ gives:
	\begin{equation}
		Z(\bar S, S) \leq |E_r| \left(R(\bar S, S) + N(\bar S, S)\right)
	\end{equation}
	Since $Z(\bar S, S) > 0$ and $R(\bar S, S) = 0$, it must be that $N(\bar S, S)$ is strictly positive.
	Thus, there is at least one arc with its tail in $\bar S$ and its head in $S$ that is being deadheaded.
	There must also be a deadheading arc with its tail in $S$ and its head in $\bar S$ because of the balance constraints~\eqref{eqn:srlc:symmetry}.
	Hence, $S$ and $\bar S$ are strongly connected, leading to a contradiction.

	Thus, the digraph $D_{\epsilon}=(V, A_{\epsilon})$ is balanced and strongly connected, i.e., the digraph is Eulerian.
	A coverage tour can be obtained from the Eulerian digraph by computing an Eulerian diwalk in $\mc O(|A_{\epsilon}|)$ computation time. 
	The cost of an Eulerian diwalk, and the corresponding coverage tour, on the digraph $D_{\epsilon}$ is $c(A_{\epsilon})=c(\tau)$, where $c(\tau)$ is the value of the objective function for a solution to the SRLC-ILP formulation.
	A solution to the SRLC-ILP formulation has a corresponding feasible coverage tour.
\end{proof}
\begin{theorem}
	Given a connected graph $G=(V, E,E_r)$, the SRLC-ILP formulation gives an optimal coverage tour.
	\label{thm:ilp}
\end{theorem}
\begin{proof}
	We first prove that any feasible coverage tour~$\tau$ has a corresponding feasible solution for the SRLC-ILP formulation with the same cost.
	In other words, the feasible solution space of the SRLC-ILP formulation includes all the feasible coverage tours.
	Represent the given coverage tour as a closed diwalk $v_1a_1\ldots a_{i-1}v_ia_i \ldots v_ka_kv_1$, along with information on whether an arc in the diwalk is serviced or deadheaded.
	For each variable $s_a, \forall a\in \mc A_r$ assign its value according to the direction in which the edge is serviced, and for each variable $d_a, \forall a \in \mc A$ assign its value equal to the number of times the arc is deadheaded.
	Also, let $D_{\epsilon}=(V, A_{\epsilon})$ be the corresponding digraph.
	Note that $c(\tau) = c(A_{\epsilon})$.
	Since the indegree of each vertex equals its outdegree in a connected closed diwalk, constraints~\eqref{eqn:srlc:symmetry} are satisfied.
	The diwalk must service each edge exactly once, satisfying~\eqref{eqn:svlc:service}.

	Without loss of generality, assume that the first arc $a_1$ in the diwalk corresponds to a required edge; we can always perform a cyclic shift of the sequence of the arcs in the diwalk to obtain an equivalent diwalk that satisfies the assumption.
	Set $v_0 = v_1$, i.e., the first vertex of the diwalk.
	Assume that the coverage tour visits $v_0$ only once; we will generalize this later.
	This implies that there are exactly two arcs in $A_{\epsilon}$ that are connected to $v_0$, one leaving and one entering.
	Set $z_a=0$ for each arc in $\mc A$. 
	Set $z_{a_1} = |E_r|$, satisfying constraint~\eqref{eqn:srlc:flowDepot}.
	Now traverse the edges in the sequence and direction given by the diwalk.
	Following the notation for a diwalk, arc $a_i$ leaves vertex $v_i$.
	During the coverage tour traversal, if a vertex $v_i$ is traversed for the first time, then set $z_{a_i}$ to be $z_{a_{i-1}}$ minus the number of arcs that are marked for servicing and enter the vertex $v_i$.
	If $v_i$ has already been traversed by an arc whose tail is $v_i$, then set $z_{a_i} $ to $z_{a_i} + z_{a_{i-1}}$.
	The value of $z_a$ remains zero for the arcs that were not traversed by the diwalk.
	Thus, constraints \eqref{eqn:srlc:flowLimit1} and \eqref{eqn:srlc:flowLimit} are satisfied for all the arcs.
	The constraints~\eqref{eqn:srlc:edgeFlow} will be satisfied at each vertex, other than $v_0$, by construction.

	Now we address the case when the coverage tour traverses the vertex $v_0$ multiple times.
	This will result in $l$ loops at~$v_0$, which we index by $j=1,\ldots,l$.
	We first perform the same procedure to assign the values of $z_{a}$ as in the preceding paragraph.
	Let the first arc (leaving $v_0$) and last arc (entering $v_0$) for a loop $j$ be $a_{j_1}$ and $a_{j_k}$, respectively.
	Now for each loop $j$, reduce the value of $z$ for all arcs in the loop $j$ by the value of $z_a$ for the last arc in the loop $a = a_{j_k}$.
	Finally, increase the value of $z$ for all arcs in any one of the loops by $|E_r| - \sum_{j\in \{1,\ldots, l\}}z_{a_{j_1}}$, to satisfy \eqref{eqn:srlc:flowDepot}.
	An example is shown in~\fgref{fig:ilp}.
	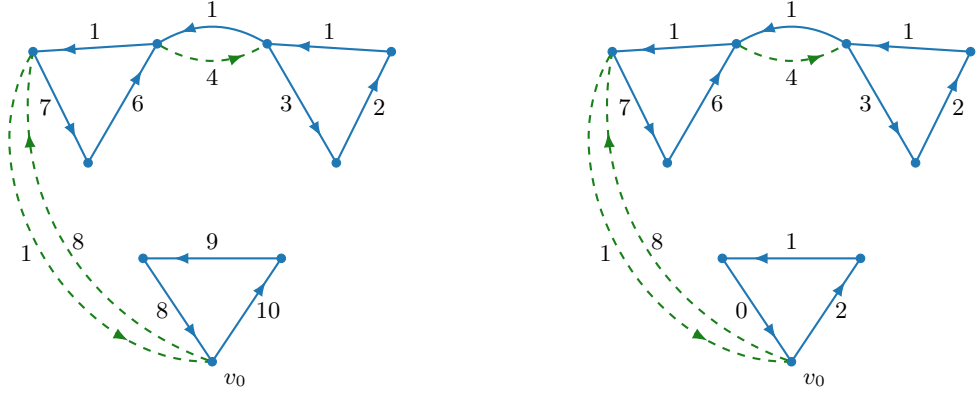
\begin{figure}[ht]
		\centering
		\input{./graphics/ilp_vis}
		\caption[Two stages of assigning flow variables]{Two stages of assigning values of $z_a$ to arcs from a given feasible coverage tour, as discussed in the proof of Theorem~\ref{thm:ilp}.
			The service and deadhead arcs are represented by solid and dashed lines, respectively.
			The arrows indicate the direction of travel.
			There are two loops connected to the depot vertex $v_0$.
			The numbers in the figures indicate values of $z_a$ in the two stages, with the right figure showing the final values.
			The $z$ values of the service arcs in the bottom triangle loop are reduced by $8$.
		Note that $z_a=0$ for all other arcs not shown in the digraph.\label{fig:ilp}}
	\end{figure}

	Hence, we can compute a feasible solution to the SRLC-ILP formulation from a feasible connected closed diwalk with the same cost, i.e., the feasible solution space has all the feasible closed diwalks (coverage tours).
	From Lemma~\ref{lem:connected}, an optimal solution to the SRLC-ILP formulation gives a feasible connected closed diwalk of the same cost.
	As the objective function of the SRLC-ILP formulation corresponds to the cost of a coverage tour, an optimal feasible solution to the formulation will also be an optimal coverage tour for the single robot line coverage problem.
\end{proof}

\subsection{Continuous Relaxation of SRLC-ILP}
We consider a continuous relaxation of the SRLC-ILP formulation that is closely related to the minimum cost flow problem.
To obtain this relaxation, we remove the connectivity constraints and relax the integer variables to make them continuous.
The relaxation, as we shall see later, gives insights into the structure of the problem and enables the development of the approximation algorithms.
The key idea is first to build a min-cost digraph, which is not necessarily balanced, and then use an LP formulation, modeled as a minimum cost flow problem, to reverse some of the service arcs in the digraph and add deadheading arcs such that the resulting digraph is balanced.
We will use the flow problem to develop approximation algorithms for the different cases of the single robot line coverage problem based on the structure of the required graph.

As before, let the input graph be $G=(V, E, E_r)$ with a required graph $G_r=(V_r, E_r)$.
Modify the SRLC-ILP formulation as follows:
\begin{enumerate}
	\item Remove the connectivity constraints~\eqref{eqn:srlc:edgeFlow}--\eqref{eqn:srlc:flowLimit} from the SRLC-ILP formulation.
		When the required graph $G_r$ is connected, the ILP formulation is still valid.
	\item Generate a digraph $D_m=(V, A_m)$ using the algorithm  \textsc{MinCostDigraph}, which selects the arc with the minimum service cost for each required edge.
		The min-cost digraph $D_m$ for an input graph is shown in~\fgref{fig:minCostDigraph}.
	\item Introduce a {\em reverse} variable $\rv{a}$ for each arc $a\in A_m$, to represent the reversal of service direction of the arc $a$.
		The reverse variables $\rv{a}$ take values from $\{0, 2\}$; $\rv a = 0$ when the service direction is not changed, and $\rv a = 2$ when the direction is reversed.
	\item Relax the integrality constraints~\eqref{eqn:srlc:trivialE} so that the deadheading variables $\dd{a_e}, \dd{\bar a_e}$, and the new reverse variables $\rv{a}$ are now continuous.
\end{enumerate}

If an arc's service direction is reversed from $a$ to ${\bar a}$, $\rv{a} = 2$, the imbalance changes by 2, and the total cost changes by $\scost{\bar a} - \scost{a}$.
We assign a {\em reversal cost} $\rcost{a}$ for each arc $a\in A_m$ and set $\rcost{a}$ to $\frac{\scost{\bar a} - \scost{a}}{2}$.

The above modifications to the ILP formulation do not lead to an exact linear relaxation of the ILP formulation, as we are removing the connectivity constraints.
However, we can formulate the resulting relaxation with continuous variables as a linear program.\\

\noindent {\bf SRLC-LP} (single robot line coverage linear program)
\begin{align}
&\text{Minimize:}&&\nonumber\\
&&&\scost{A_m} + \sum_{a\in A_m}\rcost{a} \rv{a}
+ \sum_{e\in E}\big[\dcost{a_e} \dd{a_e}+ \dcost{\bar a_e}\dd{\bar a_e}\big]  \label{eqn:lp:obj}\\
&\text{subject to:}&&\nonumber\\
&&&
\sumSr{a_1\in H(A_m, v)} -\rv{a_1}
\,+
\sumSr{b_1\in H(\mathcal A, v)}\dd{b_1}
\,	+
\sumSr{a_2\in T(A_m, v)}\rv{a_2}
\,	-
\sumSr{b_2\in T(\mathcal A, v)}\dd{b_2} =\mc I(A_m, v),\quad\forall v \in V\label{eqn:lp:symmetry}\\
&&&0\leq \rv{a} \leq  2, \quad \forall a\in A_m\label{eqn:lp:ra}\\
&&&\dd{a_e}, \dd{\bar a_e} \geq 0, \quad \forall e\in E\label{eqn:lp:dd}
\end{align}

Expression~\eqref{eqn:lp:obj} is the modified objective function.
The first term in the objective function $\scost{A_m}$ is the sum of the service costs of all the arcs in the digraph $D_m$ and is independent of the variables.
The imbalance in the digraph~$D_m$ at a vertex~$v$ is represented by~$\mc I(A_m, v)$.
The conditions~\eqref{eqn:lp:symmetry} ensure that the digraph corresponding to a feasible solution will be  balanced, i.e., the indegree will equal the outdegree at every vertex.
The constraints~\eqref{eqn:lp:ra} state that the variable~$\rv{a}$, corresponding to the reversal of service direction, is between $0$ and $2$; if~$\rv{a}=0$, the direction of travel is the same as that of the arc in~$A_m$, and if $\rv{a} = 2$, the direction of the arc is reversed, thereby reversing the service direction.

\begin{figure}[htbp]
	\centering
	\input{./graphics/minCostDigraph}
	\caption[Example min-cost digraph]{The min-cost digraph $D_m=(V, A_m)$ for the input graph given in \fgref{fig:input}.
		Note that the graph is neither balanced nor connected.
	\label{fig:minCostDigraph}}
\end{figure}
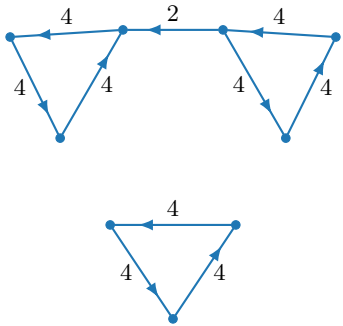

\subsection{A Network Flow Graph Model}
\label{sc:eulerian:networkFlow}
Arc routing problems are often solved by modeling them as network flow graphs and finding a minimum cost flow.
Using this approach,  algorithms for the CPP and the WPP were presented by Edmonds and Johnson~\cite{EdmondsJ73} and Win~\cite{Win89}, respectively.
Inspired by such techniques, we present a network flow graph model for solving the linear programming formulation SRLC-LP and establish its equivalence.
We will then use the model to develop approximation algorithms in \scref{sc:approx_algos}.

\begin{algorithm}[htbp]
	\Input{Graph $G=(V, E, E_r)$}
	\Output{Minimum cost digraph $D_m=(V, A_m)$}
	$A_m \gets \emptyset$\;
	\For{$e \in E_r$ }{
		\eIf{$\scost{a_e} \leq \scost{\bar a_e}$} {
			$a_e.\textsc{service}\gets \textsc{True}$;\,
		$A_m.\textsc{insert}(a_e)$\; }
		{
			$\bar a_e.\textsc{service}\gets \textsc{True}$;\,
		$A_m.\textsc{insert}(\bar a_e)$\;}
	}
	\caption{\textsc{MinCostDigraph}}
	\label{alg:mincostD}
\end{algorithm}

Let $G=(V, E, E_r)$ be the input graph.
First, generate a min-cost digraph $D_m=(V, A_m)$ using the algorithm \textsc{MinCostDigraph}($G$).
Now construct a network flow graph $D_f=(V, A_f)$, in $\mathcal O(|V| + |E|)$ time (Algorithm~\ref{alg:genFlow}),  as follows:
\begin{enumerate}
	\item For each service arc $a \in A_m$, add three arcs $a$, $\bar a$, and $a'$ to $A_f$ with the costs per unit flow $\fcost{\cdot}$ and capacities as given in Table~\ref{tb:flow}, which defines the Flow Model.
		The direction of arc $a$ is the same as that in $A_m$, whereas the direction of arcs $\bar a$ and $a'$ are opposite to that of the corresponding arc in $A_m$.
		In the flow digraph $D_f$, the arcs $a$ and $\bar a$ will represent deadheadings, and the arc $a'$ will represent service reversal.
	\item For each non-required edge $e_n \in  E\setminus E_r$, add two arcs $b$ and $\bar b$ to $A_f$, with the costs per unit flow and capacities in Table~\ref{tb:flow}.
		These two arcs, $b$ and $\bar b$, represent deadheadings of a non-required edge in the two directions.
		\begin{table*}[htbp]
			\begin{center}
				\renewcommand{\arraystretch}{1.5}
				\begin{tabular}{c@{\hskip 14pt}l@{\hskip 14pt}c@{\hskip 14pt}c}
					Arc & Description&Unit Flow Cost $\fcost{\cdot}$ & Capacity\\
					\midrule
					$a$ & Forward deadheading & $\dcost{a}$ & $\infty$ \\
					$\bar a$ & Backward deadheading & $\dcost{\bar a}$ & $\infty$ \\
					$a'$ & Service reversal & $\rcost{a'} = \big(\scost{\bar a} - \scost{a}\big)/{2}$ & $2$ \\
					$b$ & Non-required forward deadheading &$\dcost{b}$ & $\infty$ \\
					$\bar b$ & Non-required reverse deadheading & $\dcost{\bar b}$ & $\infty$ \\
					\bottomrule
				\end{tabular}
			\end{center}
			\captionsetup{justification=justified}
			\caption[Flow Model (FM) with arc costs and capacities]{Flow Model (FM): arc costs and capacities.
			Three arcs $(a, \bar a, a')$ are added for each required edge, and two arcs $(b, \bar b)$ for each non-required edge.\label{tb:flow}}
		\end{table*}
	\item For each vertex $v\in V$, assign the following node flow demand based on the degree of $v$ in $D_m$:%
		\begin{equation}
			d(v) = \outdeg(v) - \indeg(v)=\mc I(A_m, v)
			\label{eqn:flowDemand}
		\end{equation}
\end{enumerate}
Let $\flow{a}$ be the flow along the arc $a \in A_f$, and let the {\em flow vector} be $\mv f=[\flow{a}\ \mid\  a\in A_f]$.
The cost $c(\mv f)$ of a flow $\mv f$ is:%
\begin{equation}
	c(\mv f) = \sum_{ a\in A_f} \fcost{ a}\flow{a}
\end{equation}
A flow digraph $D_f$ is shown in \fgref{fig:flow}(a) for the input graph of \fgref{fig:input}, with the min-cost digraph shown in \fgref{fig:minCostDigraph}.

\begin{algorithm}[htbp]
	\Input{Graph $G=(V, E, E_r)$,\\ Digraph $D_m=(V, A_m)$}
	\Output{Flow Digraph $D_f=(V, A_f)$}
	$A_f \gets \emptyset$\;
	\For{$a\in A_m$ }{
		Insert arcs $a$, $\bar a$ and $a'$ into $A_f$, with costs and capacities from Table~\ref{tb:flow}\;
	}
	\For{$e\in E\setminus E_r$ }{
		Let $b_e$ and $\bar b_e$ be the arcs corresponding to $e$\;
		Insert arcs $b_e$ and $\bar b_e$ into $A_f$, with costs and capacities from Table~\ref{tb:flow}\;
	}
	\For{$v\in V$ }{
		$d(v) \gets \mc I(A_m, v)$\;
	}
	\caption{\textsc{ConstructFlowDigraph}}
	\label{alg:genFlow}
\end{algorithm}

We now formulate a minimum cost flow problem for the network flow graph $D_f=(V, A_f)$ and show that it models the continuous relaxation SRLC-LP.
\begin{definition}
	\label{def:optimalFlow}
	Minimum Cost Flow Problem:\\
	Let $D_f=(V, A_f)$ be a given flow digraph, along with costs, capacities, and node flow demands.
	Then the minimum cost flow problem is to find a feasible flow $\mv f$ such that:
	\begin{enumerate}
		\item the cost of flow $c(\mv f)$ is minimized, and
		\item demand $d(v)$ is satisfied for all $v\in V$.
	\end{enumerate}
\end{definition}
\begin{theorem}
	Let $G=(V, E, E_r)$ be an input graph for the single robot line coverage problem, with minimum cost digraph $D_m=(V, A_m)$ and flow digraph $D_f=(V, A_f)$.
	The minimum cost flow problem for network flow digraph $D_f$ models the continuous relaxation SRLC-LP of the SRLC-ILP formulation for the single robot line coverage problem.
	\label{thm:lp}
\end{theorem}
\begin{proof}
	Observe that any feasible solution to the minimum cost flow problem is a feasible solution to the linear programming formulation SRLC-LP, and vice versa, using the following relation between the variables:
	\begin{equation}
		\begin{aligned}
			 &f_a = \dd{a},\ f_{a'} = \rv{a},\ f_{\bar a}=\dd{\bar a}, \quad &&\forall a\in A_m\\
			 &f_{a_e} = \dd{a_e},\ f_{\bar a_e} =\dd{\bar a_e}, \quad &&\forall e\in E\setminus E_r\label{eqn:flowrelax}
		\end{aligned}
	\end{equation}
	As the capacity of the reversal arcs $\rv{a'}$ is set to $2$, the flow $f_{a'}$ across any such arc will be no greater than $2$, and the constraints~\eqref{eqn:lp:ra} are satisfied.
	The flow problem resolves the demand $d(v) = \mc I(A_m, v)$ for each vertex $v\in V$, and thus, satisfies the balance constraints~\eqref{eqn:lp:symmetry}.
	The objective of the minimum cost flow problem $c(\bf f)$ summed with $\scost{A_m}$ is then exactly the objective function~\eqref{eqn:lp:obj} of SRLC-LP.
	Thus, the theorem follows.
\end{proof}
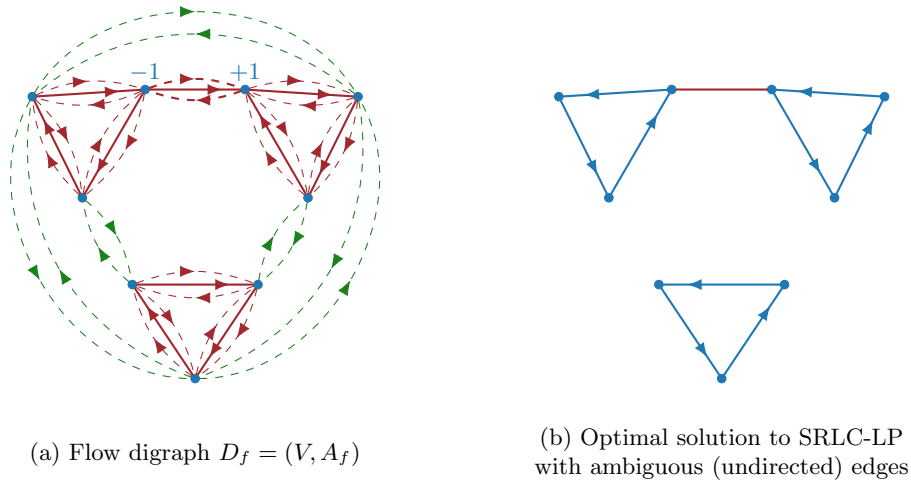
\begin{figure}[htpb]
	\centering
	\input{./graphics/flowgraph}
	\caption[Flow digraph and optimal solution for an example]{Flow digraph and an optimal solution to the SRLC-LP problem for the input graph $G$ in \fgref{fig:input} with the min-cost digraph shown in \fgref{fig:minCostDigraph}.
		(a) The red (darker) solid arcs correspond to service reversal arcs, and the red dashed arcs correspond to the deadheading of required edges.
		The dashed green (lighter) arcs are flow arcs corresponding to deadheading across non-required edges.
		The nonzero imbalances for the vertices are shown.
		(b) The set of arcs corresponding to ambiguous edges $A_u$---those that remain undirected after solving the flow problem---is shown in red (darker).
	\label{fig:flow}}
\end{figure}
\begin{remark}
	The minimum cost flow for a graph $G=(V, E)$ can be computed in time $\mc O\big((m\log n)(m +n\log n)\big)$, as shown by Orlin~\cite{Orlin93}, where $m=|E|$ and $n=|V|$.
	Optimal solutions to the minimum cost flow algorithms are integers as the imbalance at each vertex is also an  integer~(see, e.g., \cite{PapadimitriouSBook}), thus $f_{a'}\in\{0, 1, 2\}$.
	When we set $\rv a = f_{a'}$ we may have $\rv a = 1$.
	This, in turn, violates the integrality constraints~\eqref{eqn:srlc:trivialEr}, and the corresponding service variables are half-integral: $\sv{a} = \sv{\bar a} = 0.5$.
	For any coverage tour, there exists an equivalent network flow digraph with a corresponding flow.
	However, not all solutions to the network flow digraph have an equivalent coverage tour.
\end{remark}

%% file: graphics/ilp_sol.tex
	\begin{tikzpicture}[scale=0.5]
		\centering
		\small
		\clip (-6,-8) rectangle (19.8, 5.5);   
		\node[place] (1) at (0,-5)				{};
		\node[place] (2) at (1.67,-2.5)		{};
		\node[place] (3) at (-1.67,-2.5)	{};
		\node[place] (4) at (4.33, 2.5)		{};
		\node[place] (5) at (1.33, 2.69)	{};
		\node[place] (6) at (3,-0.19)			{};
		\node[place] (7) at (-4.33,2.5)		{};
		\node[place] (8) at (-3,-0.19)		{};
		\node[place] (9) at (-1.33,2.69)	{};

		\tikzstyle{scEnd}=[near end, fill=none, mDarkRed]
		\tikzstyle{scSt}=[near start, fill=none, black]
		\tikzstyle{scMid}=[midway, fill=none, black]

		\draw[rounded corners] (-6,-5.2) rectangle (-4,-2.8);
		\node[place,inner sep=0.3mm] (11) at (-5,-5) {};
		\node[place,inner sep=0.3mm] (12) at (-5,-3) {};
		\path[Arc, mDarkRed, thin] (12)  edge [bend right=20]  node[left]  {\small{$4$}} (11);
		\path[Arc, black, thin] (11)  edge [bend right=20]  node[right]  {\small{$8$}} (12);

		\draw[Edge, mBlue] (1) -- (2) node [scEnd, right] {$4$} node [scSt, right] {$8$};
		\draw[Edge, mBlue] (2) -- (3) node [scEnd, above] {$4$} node [scSt, above] {$8$};
		\draw[Edge, mBlue] (3) -- (1) node [scEnd, left]	 {$4$} node [scSt, left]  {$8$};

		\draw[Edge, mBlue] (4) -- (5) node [scEnd, above]  {$4$} node [scSt, above]  {$8$};
		\draw[Edge, mBlue] (5) -- (6) node [scEnd, left] {$4$} node [scSt, left] {$8$};
		\draw[Edge, mBlue] (6) -- (4) node [scEnd, right] {$4$} node [scSt, right] {$8$};

		\draw[Edge, mBlue] (7) -- (8) node [scEnd, left] {$4$} node [scSt, left] {$8$};
		\draw[Edge, mBlue] (8) -- (9) node [scEnd, right]  {$4$} node [scSt, right]  {$8$};
		\draw[Edge, mBlue] (9) -- (7) node [scEnd, above] {$4$} node [scSt, above] {$8$};

		\draw[Edge, mBlue] (5) -- (9) node [scEnd, above] {$2$} node [scSt, above] {$4$};
		\draw[Edge, mGreen, dashed] (2) -- (6);
		\draw[Edge, mGreen, dashed] (3) -- (8);
		\draw[Edge, mGreen, dashed] (1) to [bend right=60](4);
		\draw[Edge, mGreen, dashed] (4) to [bend right=60](7);
		\draw[Edge, mGreen, dashed] (7) to [bend right=60](1);

		\coordinate (off) at (14, 0);
		\foreach \x in {1,...,9} {
			\node[place] (\x) at ( $ (\x) +  (off) $)			{};
		}
	
		\draw[mid2arc, mBlue] (1) -- (2) node [scMid, right, mBlue] {$4$};
		\draw[mid2arc, mBlue] (2) -- (3) node [scMid, above, mBlue] {$4$};
		\draw[mid2arc, mBlue] (3) -- (1) node [scMid, left, mBlue] {$4$};

		\draw[mid2arc, mBlue] (4) -- (5) node [scMid, above, mBlue] {$4$};
		\draw[mid2arc, mBlue] (5) -- (6) node [scMid, left, mBlue] {$4$};
		\draw[mid2arc, mBlue] (6) -- (4) node [scMid, right, mBlue] {$4$};

		\draw[mid2arc, mBlue] (7) -- (8) node [scMid, left, mBlue] {$4$};
		\draw[mid2arc, mBlue] (8) -- (9) node [scMid, right, mBlue] {$4$};
		\draw[mid2arc, mBlue] (9) -- (7) node [scMid, above, mBlue] {$4$};

		\path[mBlue] (5)  edge [mid2arc, bend right]  node[scMid, above, mBlue]  {$2$} (9);
		\path[mGreen] (9)  edge [mid2arc, dashed, bend right]  node[scMid, below, mGreen]  {$2$} (5);
		\path[mGreen] (1)  edge [mid2arc, dashed, bend left=40]  node[scMid, right, mGreen]  {$1$} (7);
		\path[mGreen] (7)  edge [mid2arc, dashed, bend right=60]  node[scMid, left, mGreen]  {$1$} (1);
		\node[fill=none] at (0,-7) {\small (a) Input graph $G=(V, E, E_r)$};
		\node[fill=none, align=center] at (14.00,-7) {\small(b) Optimal coverage tour\\represented as an Eulerian digraph};
	\end{tikzpicture}

%% file: graphics/ilp_vis.tex
\begin{center}
	\begin{tikzpicture}[scale=0.55]
		\small
		\node[place, label=below right:$v_0$] (1) at (0,-5)				{};
		\node[place] (2) at (1.67,-2.5)		{};
		\node[place] (3) at (-1.67,-2.5)	{};
		\node[place] (4) at (4.33, 2.5)		{};
		\node[place] (5) at (1.33, 2.69)	{};
		\node[place] (6) at (3,-0.19)			{};
		\node[place] (7) at (-4.33,2.5)		{};
		\node[place] (8) at (-3,-0.19)		{};
		\node[place] (9) at (-1.33,2.69)	{};

		\tikzstyle{scEnd}=[near end, fill=none, mDarkRed]
		\tikzstyle{scSt}=[near start, fill=none, black]
		\tikzstyle{scMid}=[midway, fill=none, black]

		\draw[mid2arc, mBlue] (1) -- (2) node [scMid, right] {$10$};
		\draw[mid2arc, mBlue] (2) -- (3) node [scMid, above] {$9$};
		\draw[mid2arc, mBlue] (3) -- (1) node [scMid, left] {$8$};
		\path[mGreen] (1)  edge [mid2arc, dashed, bend left=40]  node[scMid, right]  {$8$} (7);
		\draw[mid2arc, mBlue] (7) -- (8) node [scMid, left] {$7$};
		\draw[mid2arc, mBlue] (8) -- (9) node [scMid, right] {$6$};
		\path[mGreen] (9)  edge [mid2arc, dashed, bend right]  node[scMid, below]  {$4$} (5);
		\draw[mid2arc, mBlue] (5) -- (6) node [scMid, left] {$3$};
		\draw[mid2arc, mBlue] (6) -- (4) node [scMid, right] {$2$};
		\draw[mid2arc, mBlue] (4) -- (5) node [scMid, above] {$1$};
		\path[mBlue] (5)  edge [mid2arc, bend right]  node[scMid, above]  {$1$} (9);
		\draw[mid2arc, mBlue] (9) -- (7) node [scMid, above] {$1$};
		\path[mGreen] (7)  edge [mid2arc, dashed, bend right=60]  node[scMid, left]  {$1$} (1);

		\coordinate (off) at (14, 0);
		\foreach \x in {2,...,9} {
			\node[place] (\x) at ( $ (\x) +  (off) $)			{};
		}
		\node[place, label=below right:$v_0$] (1) at ($ (1) + (off) $) {};

		\draw[mid2arc, mBlue] (1) -- (2) node [scMid, right] {$2$};
		\draw[mid2arc, mBlue] (2) -- (3) node [scMid, above] {$1$};
		\draw[mid2arc, mBlue] (3) -- (1) node [scMid, left] {$0$};
		\path[mGreen] (1)  edge [mid2arc, dashed, bend left=40]  node[scMid, right]  {$8$} (7);
		\draw[mid2arc, mBlue] (7) -- (8) node [scMid, left] {$7$};
		\draw[mid2arc, mBlue] (8) -- (9) node [scMid, right] {$6$};
		\path[mGreen] (9)  edge [mid2arc, dashed, bend right]  node[scMid, below]  {$4$} (5);
		\draw[mid2arc, mBlue] (5) -- (6) node [scMid, left] {$3$};
		\draw[mid2arc, mBlue] (6) -- (4) node [scMid, right] {$2$};
		\draw[mid2arc, mBlue] (4) -- (5) node [scMid, above] {$1$};
		\path[mBlue] (5)  edge [mid2arc, bend right]  node[scMid, above]  {$1$} (9);
		\draw[mid2arc, mBlue] (9) -- (7) node [scMid, above] {$1$};
		\path[mGreen] (7)  edge [mid2arc, dashed, bend right=60]  node[scMid, left]  {$1$} (1);
	\end{tikzpicture}
\end{center}

%% file: graphics/minCostDigraph.tex
\begin{center}
	\begin{tikzpicture}[scale=0.5]
		\small
		\node[place] (1) at (0,-5)				{};
		\node[place] (2) at (1.67,-2.5)		{};
		\node[place] (3) at (-1.67,-2.5)	{};
		\node[place] (4) at (4.33, 2.5)		{};
		\node[place] (5) at (1.33, 2.69)	{};
		\node[place] (6) at (3,-0.19)			{};
		\node[place] (7) at (-4.33,2.5)		{};
		\node[place] (8) at (-3,-0.19)		{};
		\node[place] (9) at (-1.33,2.69)	{};

		\tikzstyle{scEnd}=[near end, fill=none, mDarkRed]
		\tikzstyle{scSt}=[near start, fill=none, black]
		\tikzstyle{scMid}=[midway, fill=none, black]

		\draw[mid2arc, mBlue] (1) -- (2) node [scMid, right] {$4$};
		\draw[mid2arc, mBlue] (2) -- (3) node [scMid, above] {$4$};
		\draw[mid2arc, mBlue] (3) -- (1) node [scMid, left] {$4$};

		\draw[mid2arc, mBlue] (4) -- (5) node [scMid, above] {$4$};
		\draw[mid2arc, mBlue] (5) -- (6) node [scMid, left] {$4$};
		\draw[mid2arc, mBlue] (6) -- (4) node [scMid, right] {$4$};

		\draw[mid2arc, mBlue] (7) -- (8) node [scMid, left] {$4$};
		\draw[mid2arc, mBlue] (8) -- (9) node [scMid, right] {$4$};
		\draw[mid2arc, mBlue] (9) -- (7) node [scMid, above] {$4$};

		\draw[mid2arc, mBlue] (5) -- (9) node [scMid, above] {$2$};



	\end{tikzpicture}
\end{center}

%% file: graphics/flowgraph.tex
\begin{center}
	\begin{tikzpicture}[scale=0.5]
		\small
		\node[place] (1) at (0,-5)				{};
		\node[place] (2) at (1.67,-2.5)		{};
		\node[place] (3) at (-1.67,-2.5)	{};
		\node[place] (4) at (4.33, 2.5)		{};
		\node[place] (5) at (1.33, 2.69)	{};
		\node[place] (6) at (3,-0.19)			{};
		\node[place] (7) at (-4.33,2.5)		{};
		\node[place] (8) at (-3,-0.19)		{};
		\node[place] (9) at (-1.33,2.69)	{};

		\node[fill=none, above, mBlue] at (5)	{$+1$};
		\node[fill=none, above, mBlue] at (9)	{$-1$};

		\tikzstyle{scEnd}=[near end, fill=none, mDarkRed]
		\tikzstyle{scSt}=[near start, fill=none, black]

		\draw[midarc, mDarkRed, thin, dashed] (1) to [bend left=20](2);
		\draw[midarc, mDarkRed, thin, dashed] (2) to [bend left=20](1);
		\draw[midarc, mDarkRed, thin, dashed] (3) to [bend left=20](2);
		\draw[midarc, mDarkRed, thin, dashed] (2) to [bend left=20](3);
		\draw[midarc, mDarkRed, thin, dashed] (3) to [bend left=20](1);
		\draw[midarc, mDarkRed, thin, dashed] (1) to [bend left=20](3);
		\draw[mid2arc, mDarkRed] (2) -- (1);
		\draw[mid2arc, mDarkRed] (3) -- (2);
		\draw[mid2arc, mDarkRed] (1) -- (3);

		\draw[midarc, mDarkRed, thin, dashed] (5) to [bend left=20](4);
		\draw[midarc, mDarkRed, thin, dashed] (4) to [bend left=20](5);
		\draw[midarc, mDarkRed, thin, dashed] (5) to [bend left=20](6);
		\draw[midarc, mDarkRed, thin, dashed] (6) to [bend left=20](5);
		\draw[midarc, mDarkRed, thin, dashed] (4) to [bend left=20](6);
		\draw[midarc, mDarkRed, thin, dashed] (6) to [bend left=20](4);
		\draw[mid2arc, mDarkRed] (5) -- (4);
		\draw[mid2arc, mDarkRed] (6) -- (5);
		\draw[mid2arc, mDarkRed] (4) -- (6);

		\draw[midarc, mDarkRed, thin, dashed] (8) to [bend left=20](7);
		\draw[midarc, mDarkRed, thin, dashed] (7) to [bend left=20](8);
		\draw[midarc, mDarkRed, thin, dashed] (8) to [bend left=20](9);
		\draw[midarc, mDarkRed, thin, dashed] (9) to [bend left=20](8);
		\draw[midarc, mDarkRed, thin, dashed] (7) to [bend left=20](9);
		\draw[midarc, mDarkRed, thin, dashed] (9) to [bend left=20](7);
		\draw[mid2arc, mDarkRed] (8) -- (7);
		\draw[mid2arc, mDarkRed] (9) -- (8);
		\draw[mid2arc, mDarkRed] (7) -- (9);

		\draw[mid2arc, mDarkRed] (9) -- (5);
		\draw[midarc, mDarkRed, dashed] (5) to [bend left=20](9);
		\draw[midarc, mDarkRed, dashed] (9) to [bend left=20](5);

		\draw[midarc, mGreen, thin, dashed] (2) to [bend left=20](6);
		\draw[midarc, mGreen, thin, dashed] (6) to [bend left=20](2);
		\draw[midarc, mGreen, thin, dashed] (3) to [bend left=20](8);
		\draw[midarc, mGreen, thin, dashed] (8) to [bend left=20](3);
		\draw[midarc, mGreen, thin, dashed] (4) to [bend left=40](1);
		\draw[midarc, mGreen, thin, dashed] (1) to [bend right=60](4);
		\draw[midarc, mGreen, thin, dashed] (4) to [bend right=40](7);
		\draw[midarc, mGreen, thin, dashed] (7) to [bend left=60](4);
		\draw[midarc, mGreen, thin, dashed] (1) to [bend left=40](7);
		\draw[midarc, mGreen, thin, dashed] (7) to [bend right=60](1);

		\coordinate (off) at (14, 0);
		\foreach \x in {1,...,9} {
			\node[place] (\x) at ( $ (\x) +  (off) $)			{};
		}
		\draw[mid2arc, mBlue] (1) -- (2);
		\draw[mid2arc, mBlue] (2) -- (3);
		\draw[mid2arc, mBlue] (3) -- (1);

		\draw[mid2arc, mBlue] (4) -- (5);
		\draw[mid2arc, mBlue] (5) -- (6);
		\draw[mid2arc, mBlue] (6) -- (4);

		\draw[mid2arc, mBlue] (7) -- (8);
		\draw[mid2arc, mBlue] (8) -- (9);
		\draw[mid2arc, mBlue] (9) -- (7);

		\draw[Edge, mDarkRed] (5) -- (9);

		\node[fill=none] at (0,-7) {\small (a) Flow digraph $D_f=(V, A_f)$};
		\node[fill=none, align=center] at (14,-7) {\small (b) Optimal solution to SRLC-LP\\ with ambiguous (undirected) edges};
	\end{tikzpicture}
\end{center}

%% file: approx_algo.tex
\newcommand{\algLP}{\textsc{LP-Solve}}
\newcommand{\algTwoApx}{\textsc{SRLC-2Approx}}

We now present approximation algorithms for the single robot line coverage problem by partitioning it into three cases, as illustrated in \fgref{fig:alg_chart}.
The cases are based on the structure of the required graph, i.e., the subgraph induced by only the required edges:
\begin{enumerate}
	\item Eulerian required graph: The algorithm \algLP{}, derived from the network flow model, gives an optimal solution.
	\item Connected required graph: The 2-approximation algorithm \algTwoApx{} gives a solution with cost at most twice the optimal cost.
	\item General required graph: The $(\alpha(C) + 2)$-approximation algorithm, where $C$ is the number of connected components in the required graph, and $\alpha(C)$ is the approximation factor for the ATSP on a graph with $C$ vertices, gives a solution with cost at most $\alpha(C) +2$ times the optimal cost.
		When the number of connected components in the required graph is small, i.e., $C\in \mc O(\log(n))$, a 3-approximation solution is obtained.
\end{enumerate}

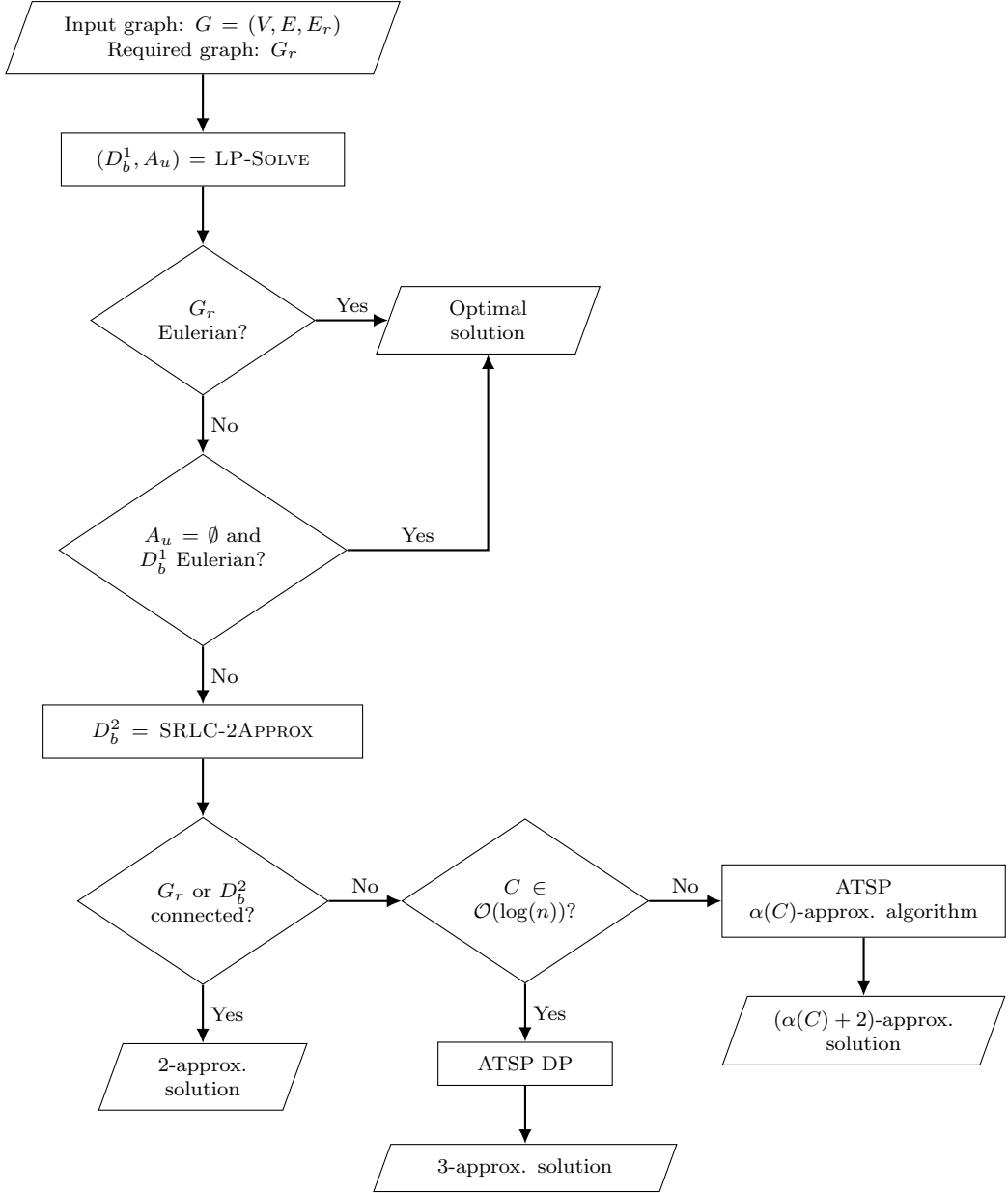
\begin{figure}[htbp]
	\centering
	\input{./graphics/alg_chart}
	\caption{Flowchart illustrating the different cases of the single robot line coverage problem based on the structure of the required graph $G_r$.
		The approximation factors of the algorithms depend on the structure of $G_r$.
		The approximation factor for the most general case is $(\alpha(C) + 2)$, where $C$ is the number of connected components in the required graph, and $\alpha(C)$ is the approximation factor for an ATSP algorithm on a graph with $C$ vertices.
	\label{fig:alg_chart}}
\end{figure}

We recast the theoretical results of the previous section in algorithmic form as Algorithm~\ref{alg:lp}, which computes a digraph by solving the relaxation SRLC-LP to the single robot line coverage problem and creates a balanced digraph $D_b=(V,A_b)$.
Given a graph $G = (V, E, E_r)$, we first compute the min-cost digraph $D_m=(V, A_m)$ (line~1) and construct the flow digraph $D_f=(V, A_f)$ from $D_m$ (line~2).
We then compute the minimum cost flow $\mv f$ for the flow digraph (line~3).
If the optimal flow across an arc $a$ is 0 or 2, we set the service direction of the corresponding edge to the direction of $a$ or $\bar a$, respectively, and add the corresponding arc to the arc set $A_b$ (lines~6--11).
For some of the required edges, the optimal flow through the reversal arcs can be~$1$, i.e., $\rv a=1$.
These arcs, denoted by $A_u$, correspond to the edges whose service direction remains {\em ambiguous} in the solution to the flow problem (lines~12--13).
Finally, we add deadheading arcs to the arc set $A_b$ according to the corresponding optimal flow through the deadheading arcs of the flow digraph (lines~14--20).
The digraph $D_b = (V, A_b)$ and the set of ambiguous arcs $A_u$ are the output of the algorithm.
Note that the digraph $D_b$ is balanced but may have multiple components, each of which is Eulerian.
An output of \algLP{} for the input graph given in \fgref{fig:input} is shown in \fgref{fig:flow}(b).
The arcs corresponding to the digraph~$D_b$ are shown in blue, whereas the ambiguous edge, corresponding to the edge set~$A_u$, is shown in dark red.
\begin{algorithm}[htbp]
	\Input{Graph $G=(V, E, E_r)$}
	\Output{Digraph $D_b=(V, A_b)$ and undirected arc set $A_u$}
	$D_m=(V, A_m) \gets$ \textsc{MinCostDigraph}($G$)\;
	$D_f=(V, A_f) \gets$ \textsc{ConstructFlowDigraph}($G, D_m$)\;
	Compute the minimum cost flow $\mv f$ for $D_f$\;
	\tcc{Generate digraph $D_b$}
	$A_b \leftarrow \emptyset$\;
	\For{$a \in A_m$}{
		\uIf{$\flow{a'} = 0$}{
			$a.\textsc{service}\gets \textsc{True}$\;
			$A_b$.\textsc{insert}\big($a$\big)\;
		}
		\uElseIf{$\flow{a'} = 2$}{
			$\bar a.\textsc{service}\gets \textsc{True}$\;
			$A_b$.\textsc{insert}\big($\bar a$\big)\;
		}
		\uElseIf{$\flow{a'} = 1$}{
			$A_u$.\textsc{insert}($a$)\;
		}
		$A_b.\textsc{insert}$($\flow{a}$ copies of $a$)\;
		$A_b.\textsc{insert}$($\flow{\bar a}$ copies of $\bar a$)\;
	}
	\For{$e \in E\setminus E_r$}{
		$A_b$.\textsc{insert}($\flow{a_e}$ copies of $a_e$)\;
		$A_b$.\textsc{insert}($\flow{\bar a_e}$ copies of $\bar a_e$)\;
	}
	\caption{\algLP}
	\label{alg:lp}
\end{algorithm}

\subsection{Eulerian Required Graph}
We first consider the case where the required graph $G_r=(V_r, E_r)$, for the input graph $G=(V, E, E_r)$, is Eulerian, i.e., the subgraph is connected, and the degree of each vertex in $V_r$ is even.
It should be noted that this special case is not the same as the Eulerian graphs for the asymmetric/windy postman problem (WPP)~\cite{RaghavachariV99SODA,Win89}.
This is because non-required edges are permitted, and the deadheading costs for the required edges can differ from their service costs.
Thus, an Eulerian tour on the required graph does not ensure a coverage tour with minimum cost.
The following lemma shows that the \algLP{} algorithm gives an optimal solution for this case in running time that is polynomial in the number of edges and vertices.

\begin{theorem}
	Let the input graph be $G=(V, E, E_r)$ such that the required graph $G_r=(V_r, E_r)$ is Eulerian.
	Then algorithm \textup{\algLP{}} gives an optimal feasible solution for the single robot line coverage problem in polynomial time.
	In particular, $\rv{a}= f_{a'} \in \{0, 2\}\;\; \forall a\in A_m\,$ and $\dd{a_e} =\flow{a_e}, \ \dd{\bar a_e}=\flow{\bar a_e} \in \mathbb N\cup \{0\}\;\;\forall e\in  E$.
	\label{lem:eulerian}
\end{theorem}
\begin{proof}
	Let the min-cost digraph be $D_m=(V, A_m)$ and the flow digraph be $D_f=(V, A_f)$, for the input graph $G=(V,E,E_r)$.
	Since the required graph $G_r$ is Eulerian, the number of edges incident at a vertex in $G_r$ is even.
	The number of outgoing arcs and incoming arcs, at a vertex in $D_m$, will both be even or will both be odd.
	Therefore, the node flow demand for a vertex $v\in V$ computed for the digraph $D_m$ will be even.
	The capacities defined in \tbref{tb:flow} are either $2$ or $\infty$.
	Furthermore, the minimum cost flow algorithm gives integral solutions.
	Hence, the optimal flow is also even for each arc $a_f\in A_f$ and, in particular, $\rv a=f_{a'}\in\{0,2\}$.
	This result can also be derived by establishing the total unimodularity of the constraint matrix obtained from the balance constraints~\eqref{eqn:lp:symmetry} by replacing each $\rv a$ by $2\tilde r_a$, $\dd{a}$ by $2\tilde{d}_{a}$, and $\dd{\bar a}$ by $2\tilde{d}_{\bar a}$.
	Note that the imbalance $\mc I(A_m, v)$ is even for each vertex $v\in V$.
	Thus, we can divide the entire equation by $2$, and the constraints will still have integral coefficients and constants.%
	\begin{equation*}
		\sumSr{a_1\in H(A_m, v)} -\tilde{r}_{a_1}
		\,+
		\sumSr{b_1\in H(\mathcal A, v)}\tilde{d}_{b_1}
		\,	+
		\sumSr{a_2\in T(A_m, v)}\tilde{r}_{a_2}
		\,	-
		\sumSr{b_2\in T(\mathcal A, v)}\tilde{d}_{b_2}
		\quad=\quad\frac{\mc I(A_m, v)}{2}\quad\forall v \in V
	\end{equation*}
	The constraint matrix for the above equation	corresponds to an incidence matrix with integers, and thus is totally unimodular.
	The substituted variables will be integral, and the original variables will all be even.
	As the required graph is connected, the digraph $D_b$ obtained from the algorithm will also be connected, and thus will be Eulerian.

	Let $m=|E|$ and $n=|V|$.
	The min-cost digraph and the flow digraph can be constructed in $\mc O(n)$ and $\mc O(m+n)$ time, respectively.
	Thus, the running time of the algorithm is dependent on the algorithm for solving the minimum cost flow problem, i.e., $\mathcal O\big((m\log n)(m +n\log n)\big)$.
\end{proof}
\subsection{Connected Required Graph}
We now consider the case where the required graph $G_r=(V_r, E_r)$, for the input graph $G=(V, E, E_r)$, is connected but not necessarily Eulerian, i.e., the degree of each vertex in $G_r$ might not be even.
The created flow digraph $D_f=(V, A_f)$ may have vertices with odd flow demands~\eqref{eqn:flowDemand}.
As a result, the optimal flow values need not be even.
While it is not a problem if $\dd{a}$ or $\dd{\bar a}$ is odd for some arc $a\in \mc A$, we need to assign service directions to the edges for which the reverse variables $r_a$ is~$1$ for arcs $a\in A_m$ and potentially add deadheading arcs to make the corresponding digraph Eulerian.

Let $A_u$ be the set of arcs for which the flow for the arc corresponding to the reversal of service direction is~$1$, i.e., $A_u=\{a\ | \ \rv{a} = \flow{a'} = 1, a\in A_m\}$, as given by algorithm \algLP{}, and let $E_u$ be the corresponding edge set.
We first check for cycles in the graph $(V,E_u)$.
Such cycles will have a sequence of edges such that the total cost of the edges if oriented in the clockwise direction is the same as the total cost of the edges if oriented in the anti-clockwise direction since the flow obtained is optimal.
We can orient such cycles in either clockwise or anti-clockwise order without changing the cost of the solution.
Now, let us say we service the edge corresponding to some $a\in A_u$ in the same direction as $a$.
This creates an imbalance  of $+1$ at the tail $t(a)$ and $-1$ at the head $h(a)$.
We can resolve this by adding the shortest deadheading path from $h(a)$ to $t(a)$, the cost of which is denoted by $\dcost{h(a), t(a)}$.
The total cost due to traversals of this edge will be the sum of the service cost of $a$ and this deadheading cost.
Similarly, we can consider servicing in the direction of $\bar a$ and then consider deadheading from $h(\bar a)$ to $t(\bar a)$.
Of the two combinations, the one that has a lower total cost is selected.
This is done for each ambiguous edge corresponding to $a\in A_u$.
This idea is described concretely in algorithm \algTwoApx{}.
Raghavachari and Veerasamy~\cite{RaghavachariV99SODA} used a linear relaxation to get a lower bound on the optimal cost for the WPP.
Using a similar lower bounding approach, we show that the \algTwoApx{} algorithm computes a coverage tour with cost at most twice that of the optimal solution.
\begin{algorithm}[htpb]
	\Input{ Graph $G=(V, E, E_r)$}
	\Output{Digraph $D_b=(V, A_b)$}
	$(D^{\text{LP}}_b=(V, A_b^{\text{LP}}), \, A_u)\gets$ \algLP{}($G$)\;
	\tcc{Let $E_u$ be the edge set corresponding to $A_u$}
	$A_b \gets A_b^{\text{LP}}$\;
	Find cycles in the graph $(V, E_u)$\;
	Orient cycles in anti-clockwise orientation and add arcs to $A_b$\;
	\For{$a \in A_u$}{
		$p\gets$ shortest deadheading path from $h(a)$ to $t(a)$\;
		$\bar p\gets$ shortest deadheading path from $h(\bar a)$ to $t(\bar a)$\;
		\eIf{$\scost{a} + \dcost{p} \leq$ $ \scost{\bar a} + \dcost{\bar p}$}
		{
			$ a.\textsc{service}\gets \textsc{True}$\;
			$A_b$.\textsc{insert}($a$)\;
		$A_b.\textsc{insert}(p)$\;}
		{
			$\bar a.\textsc{service}\gets \textsc{True}$\;
			$A_b$.\textsc{insert}($\bar a$)\;
			$A_b.\textsc{insert}(\bar p)$\;
		}
	}
	\caption{\algTwoApx{}}
	\label{alg:2opts1}
\end{algorithm}

Let the optimal flow be $\mv f$, and the cost of the optimal tour be~$c^*$.
Then the optimal value~$z^*$ of the linear program SRLC-LP is:
\begin{equation}
	z^* = \scost{A_m} + c(\mv f)
	= \scost{A_m\setminus A_u} +\scost{A_u} + c(\mv f) \leq c^*\label{eqn:cstar}
\end{equation}
Let $A_d$ denote the set of arcs with service direction and deadheading decided unambiguously by the optimal flow.
Then,
\begin{equation}
	c(A_d) = \scost{A_m\setminus A_u} + c(\mv f) - \rcost{A_u}
\end{equation}
Thus,
\begin{equation}
	\quad z^* = c(A_d) + \scost{A_u} + \rcost{A_u}\label{eqn:zstar}
\end{equation}
Substituting the value of $\rcost{A_u}=\frac{\scost{\bar{A}_u} - \scost{A_u}}{2}$ in \eqref{eqn:zstar} and using \eqref{eqn:cstar}, we have,
\begin{align}
		&2c(A_d) + \scost{A_u} + \scost{\bar{A}_u} \leq 2c^*\nonumber\\
	\text{or, } &c(A_d) + \scost{A_u} + \scost{\bar{A}_u} \leq 2c^* - c(A_d)\label{eqn:cineq}
\end{align}
\begin{theorem}
	Let the input graph be $G=(V, E, E_r)$ such that the required graph $G_r=(V_r, E_r)$ is connected.
	Then algorithm \textup{\algTwoApx{}} generates a coverage tour with cost at most twice the cost of the optimal coverage tour in polynomial time.
	\label{lem:2approx}        
\end{theorem}
\begin{proof}
	Let $A_s$ be the set of arcs corresponding to  $A_u$ with final service directions as oriented by the algorithm \algTwoApx{}.
	The total cost of the solution digraph $D_b=(V,A_b)$ generated by the algorithm is:
	\begin{equation*}
		c(A_b) = c(A_d) + \scost{A_s} + \sum_{a\in A_s}\dcost{h(a), t(a)}
		\leq c(A_d) + \scost{A_u} + \dcost{\bar A_u}
	\end{equation*}
	Note that the inequality is true because we selected the service and deadheading directions to minimize the sum of the costs for individual arcs in $A_u$.

	Furthermore, $\dcost{a} \leq \scost{a}$ for $a \in \bar A_u$ because the deadheading cost is assumed to be no greater than the corresponding service cost. Hence, 
	\begin{align}
		c(A_b) &\leq c(A_d) + \scost{{A}_u} + \scost{\bar A_u}\label{eqn:cb}
	\end{align}
	Combining \eqref{eqn:cb} with \eqref{eqn:cineq}:
	\begin{align*}
		c(A_b)\leq 2c^* - c(A_d)\leq 2c^*
	\end{align*}
	As the required graph $G_r$ is connected, the solution digraph $D_b$ is also connected.
	The digraph $D_b$ is also balanced, as discussed previously.
	Hence, a coverage tour can be generated by computing an Eulerian diwalk on $D_b$ with the same cost as that of $A_b$.
	Thus, we obtain a coverage tour of cost at most twice the cost of the optimal tour.

	The complexity of the algorithm is determined by the algorithm \algLP{}, which can be solved in $\mathcal O\big((m\log n)(m +n\log n)\big)$ time, where $m=|E|$ and $n=|V|$.
	Depending on the structure of the instance, one may need to compute the shortest deadheading paths between all pairs of vertices.
	This can be done using the Floyd-Warshall algorithm in $\mathcal{O}(n^3)$ computation time, see, e.g., \cite{DasguptaPV06book}.
\end{proof}

\subsection{General Required Graph}
We now consider input graphs for which the required graph $G_r$, induced by the required edges, may have multiple connected components.
For such graphs, algorithm \algTwoApx{} may output a disconnected digraph even though the individual connected components are Eulerian.
For the graph given in \fgref{fig:input}, with the flow digraph given in \fgref{fig:flow}, the output of algorithm \algTwoApx{} is shown in \fgref{fig:2opt}.
Note that the digraph has multiple connected components even though the individual components are Eulerian.
\begin{figure}[ht]
	\centering
	\input{./graphics/2opt}
	\caption{Digraph computed by the algorithm \algTwoApx{} for the graph in \fgref{fig:input} with the flow digraph shown in \fgref{fig:flow}.
		Note that the digraph has multiple connected components, each of which is Eulerian.
	\label{fig:2opt}}
\end{figure}
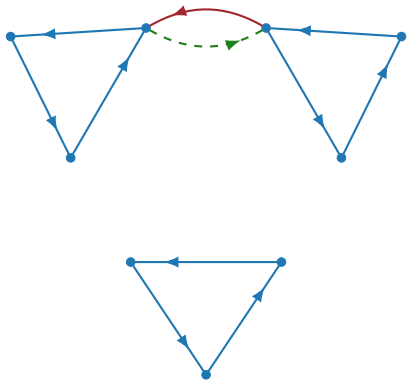

Our approach is to generate a tour through the connected components by solving the ATSP problem on an auxiliary graph whose vertices correspond to the components.
We combine the tour with the arcs generated in each component.
We develop an $\left(\alpha(C) + \beta\right)$-approximation algorithm where $C$ is the number of connected components in $G_r$ and $\beta$ is the approximation factor for the single robot line coverage problem on graphs with a connected required graph.
The $\alpha$ approximation factor depends on the approximation algorithm for the asymmetric traveling salesperson problem (ATSP), and a $\beta$ of 2 was discussed in the previous subsection using the \algTwoApx{} algorithm.
Constant factor approximation algorithms for ATSP were recently given by Svensson et al.~\cite{SvenssonTV18ATSP} and by Traub and Vygen~\cite{TraubV20ATSP}.

The output digraph $D_b$ of the \algTwoApx{} algorithm is processed to find strongly connected components.
Note that the number of strongly connected components in $D_b$ will be no greater than the number of connected components~$C$ of the required graph $G_r$ of $G$.
We then create an auxiliary graph $G_0=(V_0, E_0)$ with $V_0\subseteq V_r$ consisting of one arbitrary vertex from each connected component in $D_b$ such that a vertex $v\in V_0$ corresponds to a required edge.
The proof for the approximation factor, as we will see in Theorem~\ref{thm:lineCoverage} below, is agnostic to the choice of the required vertex in each connected component to form the vertex set~$V_0$.
We add an edge for each pair of vertices in $V_0$ to the edge set $E_0$.
For each edge $e\in E_0$, we assign two weights corresponding to the shortest deadhead cost path between the vertices of the edge in the two directions.
This makes the graph $G_0$ complete with asymmetric edge costs.%
An ATSP algorithm is then used to find a tour connecting all the vertices in $G_0$.
The arcs in the ATSP tour are then added to the disconnected diwalk generated from the \algTwoApx{} algorithm to obtain a connected coverage tour.
In the following theorem, we prove the approximation factor for our algorithm for general graphs.
The key ideas for the proof of the theorem are motivated by \cite{vanBevernKS17}.
\begin{theorem}
	The single robot line coverage problem can be solved in polynomial time with an approximation factor of $\alpha(C) + \beta$, where $\alpha(C)$ is the approximation factor for an algorithm for the asymmetric traveling salesperson problem with $C$ vertices, and $\beta$ is the approximation factor for line coverage on graphs with a connected required graph.
	\label{thm:lineCoverage}
\end{theorem}
\begin{proof}
	Let~$\tau^*$ be the optimal coverage tour and digraph~$D_b$ be the output of the \algTwoApx{} algorithm.
	Note that~$D_b$ may contain multiple strongly connected components.
	However, the number of strongly connected components in the digraph~$D_b$ will be no more than the number of connected components~$C$ in the required graph~$G_r$.
	As the relaxation SRLC-LP does not consider the connectivity constraints, the solution~$D_b$ is an approximation result to a relaxation of the original problem.
	Hence, $c(D_b) \leq \beta\, c(\tau^*)$, where $\beta = 2$ for the \algTwoApx{} algorithm.

	Let $V_0\subseteq V_r$ be a set of vertices with one arbitrary vertex from each connected component in $D_b$.
	Then $|V_0| \leq C$.
	Any coverage tour must visit each vertex in $V_0$ because each vertex in $V_0$ lies on a required edge.
	For the graph $G_0$, let $T^*$ be the optimal ATSP tour, and $T$ be the ATSP tour returned by the $\alpha(C)$-approximation algorithm.
	Then $c(T) \leq \alpha(C)\, c(T^*)\leq \alpha(C)\, c(\tau^*)$.

	Let $\tau$ be the final coverage tour obtained by adding the arcs from $T$ to the digraph $D_b$ and generating an Eulerian tour.
	Then $c\left(\tau\right) = c\left(D_b\right) + c\left(T\right) \leq  \beta\; c(\tau^*) + \alpha(C)\, c(\tau^*)=\left(\alpha(C) + \beta\right)\; c(\tau^*)$.
\end{proof}

\begin{corollary}
	Combining Theorem~\ref{lem:2approx} and Theorem~\ref{thm:lineCoverage}, and noting that the number of connected components $C$ is usually small in practice, as stated by~\cite{vanBevernKS17}, we observe:
	\begin{enumerate}
		\item The single robot line coverage problem has an $\alpha(C) + 2$ approximation factor.
			This also improves the previously best-known approximation result of $\alpha(C) + 3$ for the asymmetric rural postman problem given by~\cite{vanBevernKS17}.
		\item If $C\in \mathcal O(\log n)$, an $\mathcal O(C^2 2^C)$ dynamic programming algorithm gives the optimal ATSP solution in polynomial time, giving a 3-approximation algorithm for the single robot line coverage problem.
	\end{enumerate}
\end{corollary}
\begin{remark}
	For the special case when we have two connected components, we do not have an ATSP tour.
	In such a scenario, we can duplicate one of the vertices $v\in V_0$ and add a zero-cost edge from the duplicated vertex to $v$.
	The edge set $E_0$ will then be created on these three vertices.
\end{remark}

The final coverage tour for the example graph, given in \fgref{fig:input}, is shown in \fgref{fig:finalSol}.
For this example, the cost of the tour obtained using the approximation algorithm is optimal.
\begin{figure}[ht]
	\centering
	\input{./graphics/finalSol}
	\caption{%
		The final coverage tour, in the form of an Eulerian digraph, obtained for the input graph in \fgref{fig:input}.
		The algorithm generates an ATSP tour on the solution from \algTwoApx{}, shown in \fgref{fig:2opt}.
		For each connected component, an arbitrary vertex is selected, shown as red unfilled circles.
		As there are two connected components, the algorithm creates an auxiliary vertex by duplicating one of the selected vertices.
		A different choice for these vertices can result in a different coverage tour, potentially of a different cost. 
		The total cost of the coverage tour is 42, which is the same as the optimal cost.
	\label{fig:finalSol}}
\end{figure}
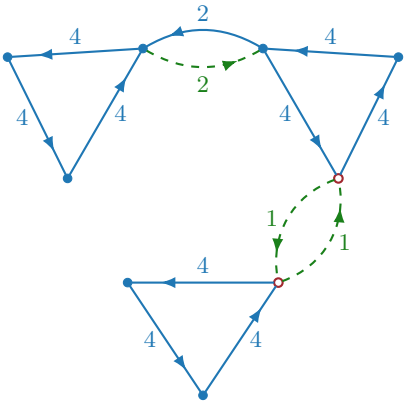

\subsection{Improvements and Extensions}
We now provide two heuristics that improve the quality of the solutions generated by the algorithms discussed in the paper.
We explore the use of the generalized traveling salesperson problem (GTSP) instead of the ATSP.
We discuss the practical aspects of implementing the algorithms.
Finally, we discuss the application of our algorithm to the line coverage problem with multiple robots.

{\bf Short-circuiting:}
Our first heuristic is to improve the final coverage tour by replacing a sequence of consecutive deadheading edges with the shortest path from the first vertex of the sequence to the last vertex of the sequence.

{\bf 2-opt heuristic:}
The quality of the solution can be further improved by employing a simple 2-change local neighborhood search, also known as 2-opt, similar to that for the TSP, see, e.g., \cite{AlgorithmIlluminated4}.
Since it is an {\em anytime} heuristic, i.e., it always maintains a feasible solution, we can ensure that the total number of constant-time local moves is restricted to $n^3$ to maintain the $\mathcal O(n^3)$ running time of the algorithm.
We discuss the computational costs and the improvement in the solutions on a dataset of 50 road networks in \scref{sc:sim}.

{\bf GTSP based algorithm:}
In the algorithm for the case of general required graphs, an arbitrary vertex was selected for each connected component to create an auxiliary graph $G_0$ required as input to the ATSP algorithm.
Since the choice of the vertices may affect the cost of the tour, an alternative technique is to formulate the problem as a generalized traveling salesperson problem (GTSP).
Each connected component forms a cluster, and the vertices in the connected component form nodes in the cluster.
The cost between any pair of nodes corresponds to the shortest deadheading path between the nodes.
The GTSP is then to compute a minimum cost tour such that at least one of the nodes in each cluster is visited.
A GTSP instance can be solved by converting it to an instance of the ATSP with $n$ vertices, where $n$ is the total number of nodes in the GTSP instance, as given by \cite{NoonB93}.
Such a solution has an approximation factor of $\alpha(n)+2$, where $\alpha(n)$ is the approximation factor of an algorithm for the ATSP on a graph with $n$ vertices.
In principle, a GTSP based algorithm can provide better solutions as we no longer select an arbitrary vertex from each connected component.
However, in practice, the GTSP based algorithm may require a longer computation time and, therefore, is not always suitable for robotics applications that require rapid computation of the coverage tours for the robots.

{\bf Practical considerations:}
The $(22+\epsilon)$-approximation algorithm for the ATSP given by Traub and Vygen~\cite{TraubV20ATSP} is not very practical for robotics applications because of its running time and challenging implementation.
However, it is very relevant for providing a constant-factor approximation guarantee.
As the number of connected components is usually very small, the dynamic programming based algorithm of Held and Karp~\cite{HeldK62} works well in practice.
The algorithm runs in $\mathcal{O}(n^2 2^n)$, where $n$ is the number of vertices.
Techniques for bitwise operations by Knuth~\cite{Knuth2011TACP} are used to run through the $2^n$ combinations.
We also use a state-of-the-art solver for ATSP from Helsgaun~\cite{Helsgaun00} for instances with a larger number of connected components.
The results are discussed in Section~\ref{sc:sim}.

{\bf Multiple robots:}
Our approximation algorithms have implications for the algorithms for arc routing problems with multiple robots.
In the capacitated arc routing problem (CARP), the edges of the graph have a demand associated with them, and the robots have a capacity $Q$, see, e.g., \cite{CorberanL14}.
The task is to find a set of tours for a team of $k$ robots such that the total demand for any of the robots does not exceed its capacity $Q$.
The objective function is the total cost of all the tours for the robots.
One of the strategies for the CARP and its variants is first to find a large tour ignoring the demand constraints by employing algorithms for the single robot problem.
Then the tour is split into smaller components that respect the capacity constraints \cite{Wohlk08}.
Thus, any improvements to the algorithms for the single robot version improve the quality of the solution for the version with multiple robots.
A tour-splitting algorithm was given by van Bevern et al.~\cite{vanBevernKS17} for the CARP on mixed and windy graphs.
The algorithm has an approximation factor of $8\alpha(C+1) +27$ in general and an approximation factor of $35$ when the number of connected components $C$ is small, i.e., $C\in \log(n)$.
Our results immediately improve these approximation factors to $8\alpha(C+1) + 19$ and $27$, respectively.

%% file: graphics/alg_chart.tex
\begin{tikzpicture}[scale=0.5]
	\footnotesize
	\tikzstyle{arrow} = [thick,->,>=stealth]
	\tikzstyle{scEnd}=[near end, fill=none, mDarkRed]
	\tikzstyle{scSt}=[near start, fill=none, black]
	\tikzstyle{scMid}=[midway, fill=none, black]
	\def\arrlen{0.8}
	\node (in) [ioNode,text width=4.5cm] {Input graph: $G=(V, E, E_r)$\\ Required graph: $G_r$};
	\node (lp) [processNode, below=\arrlen cm of in, text width=3.5cm] {$(D^1_b, A_u)=\algLP{}$};
	\draw[Arc] (in) -- (lp);
	\node (euler) [decisionNode, below=\arrlen cm of lp, text width=1.5cm] {$G_r$ Eulerian?};
	\draw[Arc] (lp) -- (euler);
	\node (outEuler) [ioNode,text width=2.2cm, right=1cm of euler] {Optimal solution};
	\draw [Arc] (euler) -- node[scMid, above] {Yes} (outEuler);
	\node (eulerb) [decisionNode, below=\arrlen cm of euler, text width=2.2cm] {$A_u=\emptyset$ and\\ $D^1_b$ Eulerian?};
	\draw [Arc] (euler) -- node[scMid, right] {No} (eulerb);
	\draw [Arc] (eulerb) -| node[scSt, above, black] {Yes} (outEuler);
	\node (conn) [processNode, below=\arrlen cm of eulerb, text width=4cm] {$D^2_b= \algTwoApx{}$};
	\draw [Arc] (eulerb) -- node[scMid, right] {No} (conn);
	\node (connq) [decisionNode, below=\arrlen cm of conn, text width=1.8cm] {$G_r$ or $D^2_b$ connected?};
	\draw[Arc] (conn) -- (connq);
	\node (two) [ioNode,text width=2cm, below=\arrlen cm of connq] {$2$-approx.\\solution};
	\draw [Arc] (connq) -- node[scMid, right] {Yes} (two);
	\node (C) [decisionNode, right=1cm of connq, text width=1.7cm] {$C\in \mathcal O(\log(n))$?};
	\draw [Arc] (connq) -- node[scMid, above] {No} (C);
	\node (dp) [processNode, below=\arrlen cm of C, text width=2cm] {ATSP DP};
	\node (three) [ioNode,text width=3.5cm, below=\arrlen cm of dp] {$3$-approx. solution};
	\draw [Arc] (C) -- node[scMid, right] {Yes} (dp);
	\draw [Arc] (dp) -- (three);
	\node (capprox) [processNode, right=1cm of C, text width=3.5cm] {ATSP\\$\alpha(C)$-approx. algorithm};
	\node (plusTwo) [ioNode,text width=3cm, below=\arrlen cm of capprox] {$\left(\alpha(C) +2\right)$-approx.\\solution};
	\draw [Arc] (C) -- node[scMid, above] {No} (capprox);
	\draw [Arc] (capprox) -- (plusTwo);
\end{tikzpicture}

%% file: graphics/2opt.tex
\begin{center}
	\begin{tikzpicture}[scale=0.6]
		\clip (-6.2,-5.2) rectangle (5, 3.2);   
		\node[place] (1) at (0,-5)				{};
		\node[place] (2) at (1.67,-2.5)		{};
		\node[place] (3) at (-1.67,-2.5)	{};
		\node[place] (4) at (4.33, 2.5)		{};
		\node[place] (5) at (1.33, 2.69)	{};
		\node[place] (6) at (3,-0.19)			{};
		\node[place] (7) at (-4.33,2.5)		{};
		\node[place] (8) at (-3,-0.19)		{};
		\node[place] (9) at (-1.33,2.69)	{};

		\draw[mid2arc, mBlue] (1) -- (2);
		\draw[mid2arc, mBlue] (2) -- (3);
		\draw[mid2arc, mBlue] (3) -- (1);

		\draw[mid2arc, mBlue] (4) -- (5);
		\draw[mid2arc, mBlue] (5) -- (6);
		\draw[mid2arc, mBlue] (6) -- (4);

		\draw[mid2arc, mBlue] (7) -- (8);
		\draw[mid2arc, mBlue] (8) -- (9);
		\draw[mid2arc, mBlue] (9) -- (7);

		\draw[mid2arc, mDarkRed] (5) to [bend right] (9);
		\draw[mid2arc, mGreen, dashed] (9) to [bend right] (5);
	\end{tikzpicture}
\end{center}

%% file: graphics/finalSol.tex
\begin{center}
	\begin{tikzpicture}[scale=0.6]
		\small
		\node[place] (1) at (0,-5)				{};
		\node[place, fill=none, thick, mDarkRed] (2) at (1.67,-2.5)		{};
		\node[place] (3) at (-1.67,-2.5)	{};
		\node[place] (4) at (4.33, 2.5)		{};
		\node[place] (5) at (1.33, 2.69)	{};
		\node[place, fill=none, thick, mDarkRed] (6) at (3,-0.19)			{};
		\node[place] (7) at (-4.33,2.5)		{};
		\node[place] (8) at (-3,-0.19)		{};
		\node[place] (9) at (-1.33,2.69)	{};

		\tikzstyle{scEnd}=[near end, fill=none, mDarkRed]
		\tikzstyle{scSt}=[near start, fill=none, black]
		\tikzstyle{scMid}=[midway, fill=none, black]

		\draw[mid2arc, mBlue] (1) -- (2) node [scMid, right, mBlue] {$4$};
		\draw[mid2arc, mBlue] (2) -- (3) node [scMid, above, mBlue] {$4$};
		\draw[mid2arc, mBlue] (3) -- (1) node [scMid, left, mBlue] {$4$};

		\draw[mid2arc, mBlue] (4) -- (5) node [scMid, above, mBlue] {$4$};
		\draw[mid2arc, mBlue] (5) -- (6) node [scMid, left, mBlue] {$4$};
		\draw[mid2arc, mBlue] (6) -- (4) node [scMid, right, mBlue] {$4$};

		\draw[mid2arc, mBlue] (7) -- (8) node [scMid, left, mBlue] {$4$};
		\draw[mid2arc, mBlue] (8) -- (9) node [scMid, right, mBlue] {$4$};
		\draw[mid2arc, mBlue] (9) -- (7) node [scMid, above, mBlue] {$4$};

		\path[mBlue] (5)  edge [mid2arc, bend right]  node[scMid, above, mBlue]  {$2$} (9);
		\path[mGreen] (9)  edge [mid2arc, dashed, bend right]  node[scMid, below, mGreen]  {$2$} (5);
		\path[mGreen] (2)  edge [mid2arc, dashed, bend right=40]  node[scMid, right, mGreen]  {$1$} (6);
		\path[mGreen] (6)  edge [mid2arc, dashed, bend right=40]  node[scMid, left, mGreen]  {$1$} (2);
	\end{tikzpicture}
\end{center}

%% file: simulations.tex

\newcommand{\BtwoA}{$\beta2$-ATSP}
\newcommand{\BtwoAtwo}{$\beta$2-ATSP-2opt}
\newcommand{\BthreeA}{$\beta$3-ATSP}
\newcommand{\BtwoG}{$\beta$2-GTSP}

We now establish the efficiency and efficacy of the presented algorithms for the single robot line coverage problem through simulations and experiments.
The algorithms are implemented in \texttt{C++} and executed on a desktop computer with an Intel Core i9-7980XE processor.
We take advantage of the advances in linear programming solvers and use Gurobi~\cite{Gurobi} to obtain solutions rapidly for the minimum cost flow problem.
Smaller instances of the ATSP are solved using the dynamic programming algorithm given by Held and Karp~\cite{HeldK62}, while larger instances are solved using the LKH solver from Helsgaun~\cite{Helsgaun00}.
The GLKH solver by Helsgaun~\cite{Helsgaun15} is used to solve the GTSP instances.
The algorithm for the RPP on mixed and windy graphs by van Bevern et al.~\cite{vanBevernKS17} is also adapted for the single robot line coverage problem for comparison with the algorithms presented in this paper.
The short-circuiting based tour improvement routine is applied to the solutions generated from each algorithm as it replaces consecutive deadheadings with the shortest deadheading paths.
The SRLC-ILP formulation is solved using Gurobi and executed on a cluster node with 48 cores.
The solutions from the approximation algorithm developed in this paper are used to provide an initial solution to the ILP formulation, which helps in upper-bounding the branch-and-bound algorithms and obtaining solutions faster.
Solving an ILP to obtain an optimal solution can take a long time; it took around 20 hours for one of the instances with 635 vertices and 730 required edges.

\subsection{Simulation Results on Road Networks}
An important application of the single robot line coverage problem is the mapping, inspection, and surveillance of road networks.
We generated a dataset\footnote{The dataset and a tool for extracting road network data are available at: \\ \url{https://github.com/UNCCharlotte-CS-Robotics/LineCoverage-dataset}.} consisting of road networks from the $50$ most populous cities in the world.
These road networks differ considerably in structure from one another, and thus allow testing of the algorithms on a variety of graphs.
The data was obtained from OpenStreetMap~\cite{OpenStreetMap} by selecting a bounding polygon of 0.5 to 2.0\SIkmsqr{} area using a web-based tool.
The dataset consists of road networks with 75 to 635 vertices and 93 to 730 required edges.
As UAVs can fly from one location to another, non-required edges are added between each pair of vertices, resulting in tens of thousands of non-required edges.
In the case of no-fly zones, the corresponding non-required edges can be pruned in practice.
The servicing and deadheading speeds are set to 7\SIvel{} and 10\SIvel{}, respectively.
A wind of 2\SIvel{} is simulated from the south-west direction, i.e., $\pi/4$ radians from the horizontal axis.
These parameters are set based on real-world experiments discussed in the following subsection.

Denote the speed of the UAV by~$v$ and the wind speed by~$w$.
For the traversal of an edge from a tail vertex $t$ to a head vertex $h$, let the travel vector $\mv{t}$ denote the vector from $t$ to $h$.
Let $\phi$ be the angle between the wind vector and the travel vector $\mv t$ for an edge.
Then the effective speed of the UAV is given by:
\begin{align}
	v_{\text{eff}} = w \cos{\phi} + \sqrt{v^2 - w^2 \sin^2{\phi}}\label{eqn:traveltime}
\end{align}
The cost function is defined as the time taken for the UAV to traverse an edge:
\begin{align}
	c(t, h) = \frac{\norm{\mv t}_2}{v_\text{eff}}\label{eqn:traveltimea}
\end{align}
Here, $\norm{\mv t}_2$ is the Euclidean distance from $t$ to $h$.
The speed~$v$ of the UAV is set to the servicing or deadheading speed according to its travel mode.
Note that the cost function is asymmetric due to wind.

We use the following notation for brevity:
\begin{enumerate}
	\item \BtwoA{}: Algorithm \algTwoApx{} along with the dynamic programming algorithm for the ATSP.
	\item \BtwoG{}: Algorithm \algTwoApx{} along with the GLKH solver for the generalized ATSP.
	\item \BtwoAtwo{}: Algorithm \algTwoApx{} along with the dynamic programming algorithm for the ATSP and 2-opt improvement heuristic.
	\item \BthreeA{}: Algorithm given by van Bevern et al.~\cite{vanBevernKS17} along with the dynamic programming algorithm for the ATSP.
\end{enumerate}

\fgref{srlc:fig:sample} shows four of the fifty road networks in the dataset, along with the coverage tours obtained from the ILP formulation and the \BtwoAtwo{} algorithm presented in the paper.
\begin{figure}[htbp]
	\centering
	\includegraphics[width=0.94\textwidth]{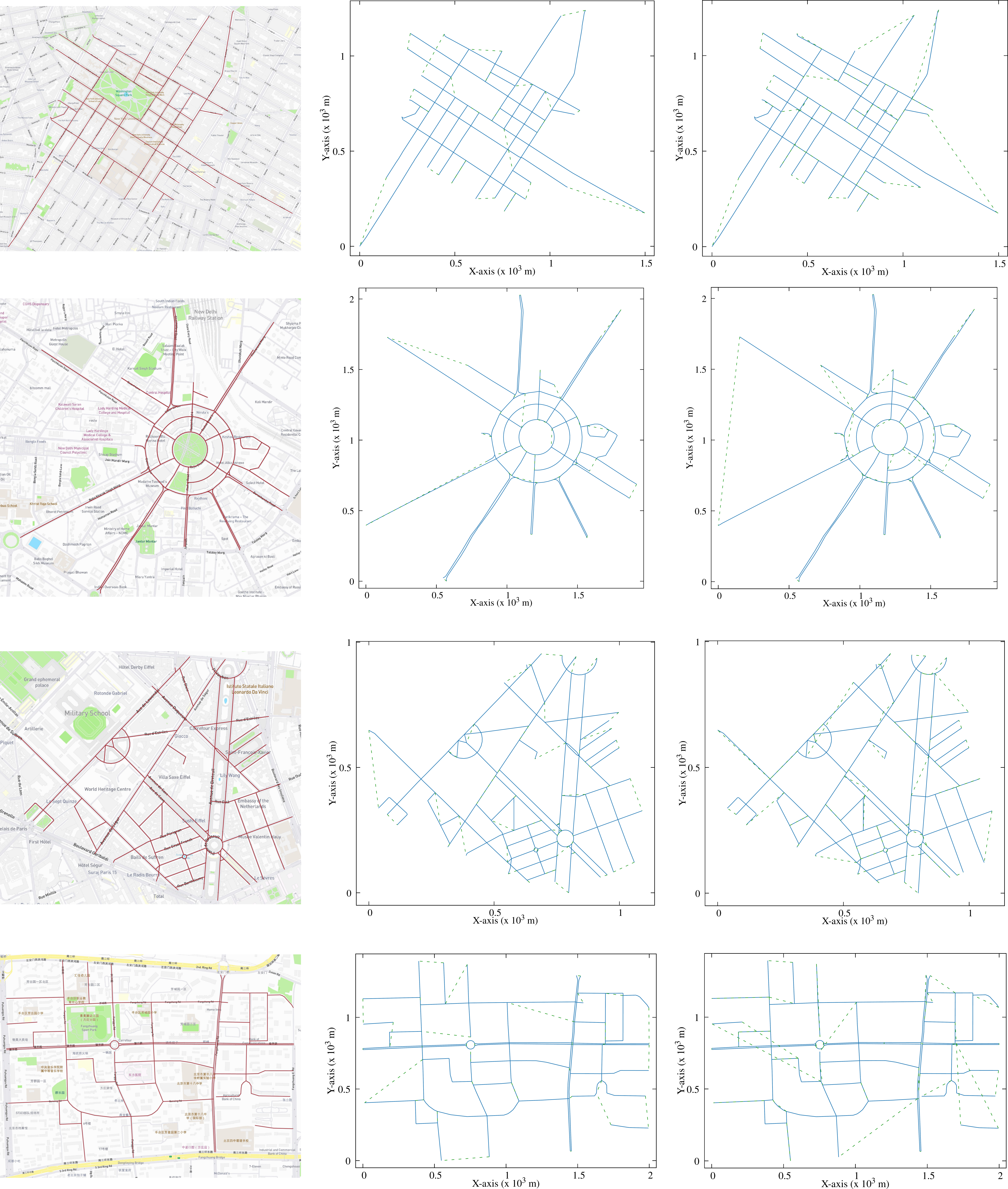}
	\caption{Four of the fifty sample road networks obtained from the most populous cities: The first column is the map representing the input graph, the second column is the optimal solution obtained using the SRLC-ILP formulation, and the third column is the final result of the algorithm (\BtwoAtwo{}) presented in this paper.
		The road networks, from top to bottom, are from (a) New York, (b) Delhi, (c) Paris, and (d) Beijing.
		Only the required edges representing the road network are shown on the map.
		There is a non-required edge for each pair of vertices in the graph.
		For example, the New York dataset has one connected component, $379$ vertices, $402$ required edges, and $71,631$ non-required edges.
		The cost of the optimal coverage tour (middle column) generated by the ILP formulation is $2018.69$, whereas the cost of the solution (right column) computed by the \BtwoAtwo{} algorithm is $2199.82$.
		The solid blue lines represent servicing, while the dashed green lines represent deadheading travel.
	\label{srlc:fig:sample}}
\end{figure}

A cost comparison of the solutions obtained from the \BtwoA{}, the \BtwoAtwo{}, and the \BthreeA{} algorithms is shown in \fgref{fig:cost}.
The $y$-axis shows the percentage difference in cost with respect to the optimal solution, i.e., $\frac{c-c^*}{c^*}\times 100$, where $c$ is the cost of the coverage tour given by the corresponding algorithm, and $c^*$ is the optimal solution obtained using the ILP formulation.
The solutions obtained by our final algorithm \BtwoAtwo{} are within $10$\% of the optimal solution.
Algorithm \algTwoApx{} with the DP algorithm for ATSP, denoted by \BtwoA{}, generally performs better than the algorithm given by van Bevern et al.~\cite{vanBevernKS17}, denoted by \BthreeA{}.
There seems to be no perceptible trend in cost difference percentage with the increase in the instance size.
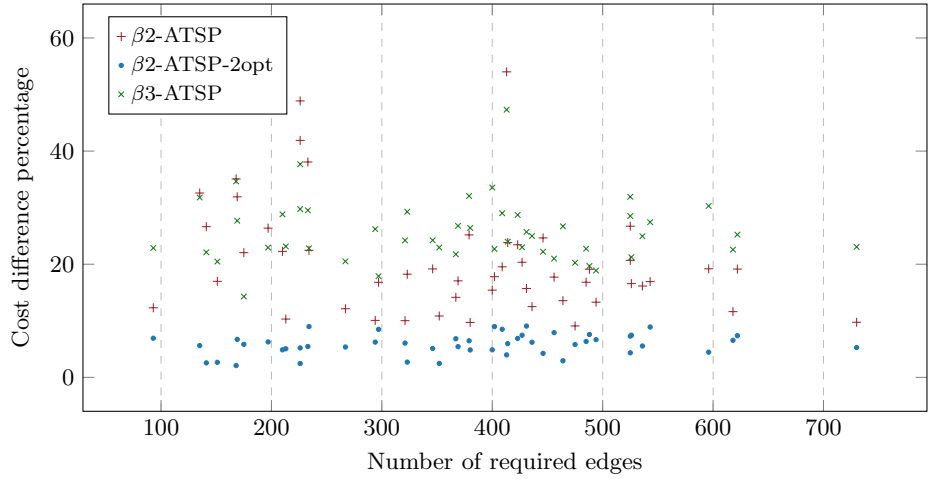
\begin{figure}[htbp]
	\centering
	\input{./graphics/costs_b2}
	\caption[Cost comparison of the approximation algorithm with ILP formulation]{%
		Cost comparison of the algorithms for the $50$ road network dataset:
		The \BtwoA{} algorithm, shown by red plus $\color{mDarkRed}+$ markers, generally performs better than the \BthreeA{} algorithm given by van Bevern et al.~\cite{vanBevernKS17}, shown by green cross markers $\color{mGreen} \times$.
		The solutions obtained by the \BtwoAtwo{} algorithm, shown by blue dots $\color{mBlue}\bullet$, are consistently within $10$\% of the optimal.
	\label{fig:cost}}
\end{figure}
\begin{figure}[htbp]
	\centering
	\input{./graphics/time_b2}
	\caption[Computation time comparison of various algorithms]{%
		Computation time comparison of various algorithms in the paper:
		The \BtwoA{} algorithm takes time comparable to the \BthreeA{} algorithm.
		The 2-opt heuristic improvement over the \BtwoA{} algorithm takes very small additional time for small instances and up to $2$\,s for larger instances.
		The running times are averaged over $100$ runs.
	\label{fig:time}}
\end{figure}
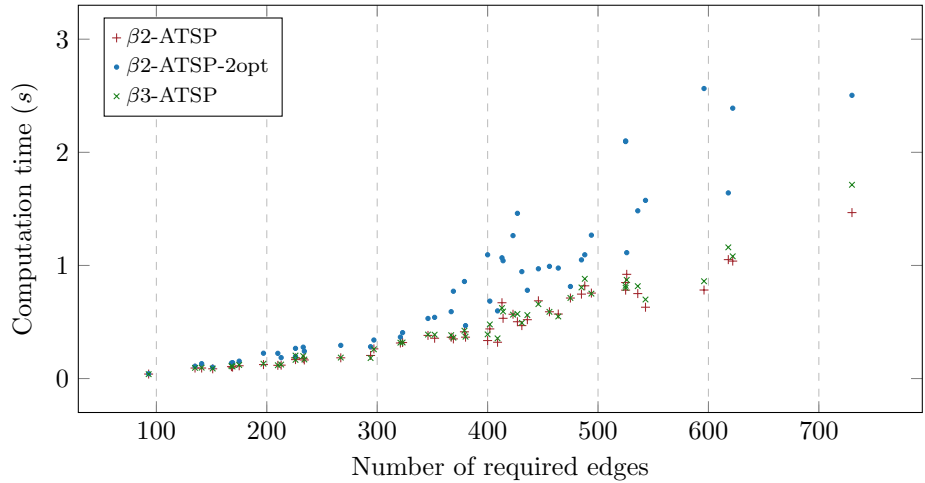

The computation times, shown in \fgref{fig:time}, were obtained by averaging over $100$ runs.
The algorithm \BtwoA{} is comparable in running time to \BthreeA{} while providing better solutions, in general.
For the \BtwoAtwo{} algorithm, additional time is spent to run the 2-opt heuristic, which runs very fast for smaller instances and takes up to an additional $2$\,s for some of the larger instances.
All $50$ instances were solved within $3$\,s with a mean time of 0.83\,s and a median time of 0.73\,s, using the \BtwoAtwo{} algorithm.

Road networks, for which the solution obtained by the \algTwoApx{} algorithm has more than one connected component, require the additional step of selecting the vertices to form the auxiliary graph $G_0$.
We select the first indexed required vertex in each connected component to form the auxiliary graph $G_0$.
Hence, the process is deterministic for a given input graph.
In general, one could randomly select a required vertex from each connected component.
Note that the approximation factor is not affected by the selection of the vertices in $G_0$, as shown in Theorem~\ref{thm:lineCoverage}.
The \BtwoG{} algorithm uses GTSP, and thus considers all vertices in each connected component.
\fgref{fig:atspvsgtsp} shows comparisons of costs and computation time for the \BtwoA{} and the \BtwoG{} algorithms.
The comparison is performed for the instances with at least three connected components in the required graph.
Using the GTSP gives better solutions in general, as the algorithm has the flexibility to select any vertex in each of the connected components.
In contrast, an arbitrary vertex is selected for each connected component for the \BtwoA{} algorithm.
In only one of the instances (Ahmedabad), the final coverage tour obtained using the DP algorithm for the ATSP resulted in slightly better results than that for using the GTSP, and this is because the ATSP based solution was more favorable for the short-circuiting routine and resulted in a better solution overall.
The \BtwoA{} algorithm is computationally much more efficient.

The simulation results\footnote{Results are available in the repository:\url{https://github.com/UNCCharlotte-CS-Robotics/LineCoverage-dataset}.} indicate that all the ATSP based algorithms are sufficiently fast.
With a small additional computational cost for the 2-opt heuristic, the \BtwoAtwo{} algorithm computes high-quality solutions. 

\begin{figure}[ht]
	\centering
	\input{./graphics/atspvsgtsp}
	\caption[Computation time and cost comparisons using generalized ATSP as subroutine]{%
		Computation time and cost comparisons for the \BtwoA{} and the \BtwoG{} algorithms:
		The \BtwoG{} algorithm gives better solutions, in general.
		However, the \BtwoA{} algorithm is computationally much more efficient.
		The comparison is performed for the instances with at least three connected components in the required graph.
	\label{fig:atspvsgtsp}}
\end{figure}
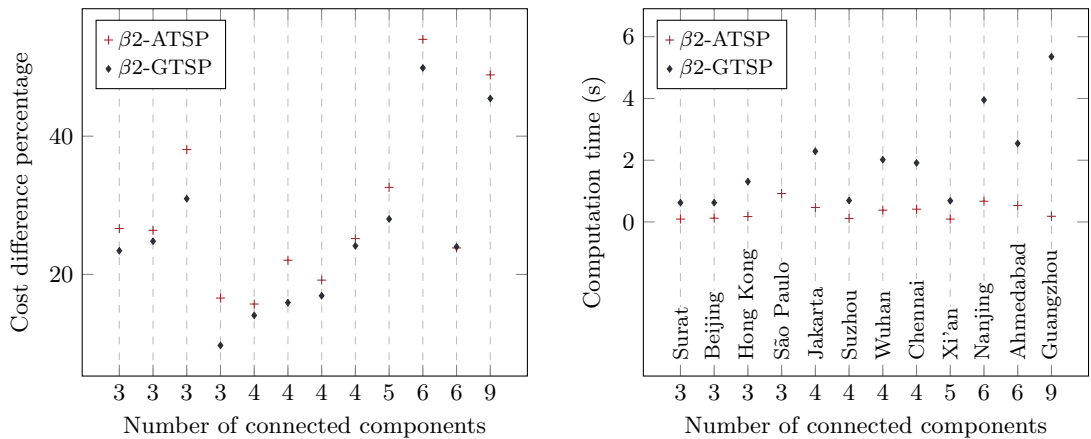

\subsection{Experiments with UAVs on Road Networks}
We performed line coverage on two different portions of the UNC Charlotte road network using a DJI Phantom~4 quadrotor UAV.
Figure~\ref{fig:uncc_example} shows a portion of the road network, and \fgref{fig:lot_input} shows a network of lanes on a set of parking lots. 
The experiments were performed with two different sets of operating conditions.
The servicing and deadheading speeds, along with wind speeds and directions, are specified in Table~\ref{tb:exp}.
The cost functions are based on the time to traverse the respective edge; the wind conditions make the costs asymmetric, as specified by Equations~\eqref{eqn:traveltime} and~\eqref{eqn:traveltimea}.
The computed line coverage tour costs using the SRLC-ILP formulation and the \BtwoAtwo{} algorithm are provided in Table~\ref{tb:exp}.
The table also provides the actual flight times.
Figures~\ref{fig:uncc} and~\ref{fig:lot} show the computed coverage tours using the \BtwoAtwo{} algorithm, the actual flight paths, and orthomosaics for the two datasets.
We generated orthomosaics from the images collected during the flights.
The images are taken only during servicing (and not during deadheading), leading to a smaller number of images and reducing the time to compute the orthomosaic.

We have the following observations from our experiments.%
\begin{enumerate}
	\item The actual flight time differs from the computed flight time.
		Since we use a commercial mobile phone application to fly the UAV autonomously along the coverage tour, we do not have access to a model of the controller.
		As our formulation allows arbitrary cost functions, a high-fidelity model of the trajectory controller and wind effects can be incorporated for better results.
		Another aspect is that we do not model turning costs, and UAVs need to slow down to take sharper turns.
		This increases the actual flight time and indicates the importance of modeling the effect of turns in the objective function in the future. 
	\item Since we flew the UAV at a relatively high altitude (compared to the distance between parallel required edges and the sensor field of view), the generated orthomosaic provides an area coverage of the parking lots.
		Line coverage can, in fact, be used as a subroutine for area coverage~\cite{AgarwalA22RAL}.
\end{enumerate}

In practice, UAVs are generally launched from an elevated position to maintain line of sight.
Similarly, there may be additional physical and safety constraints that restrict the launch location of robots.
Such a location need not be part of the road network, as shown by the blue marker in Figure~15(b), and induce additional deadheading travel to and from the launch location.
Although the costs of deadheadings from the launch location have not been considered in the experiments, one can easily incorporate them by adding an artificial required edge with zero service costs such that both the vertices of the edge correspond to the launch location.
The two experiments demonstrate the use of our line coverage formulation and the algorithms to generate efficient coverage tours for linear infrastructure.
The two modes of travel---servicing and deadheading---can be conveniently modeled in the formulation allowing lower operation times.
Furthermore, allowing deadheading reduces the amount of sensor data required for analysis.
\begin{table}[htbp]
	\begin{center}
		\renewcommand{\arraystretch}{1.8}
		\centering
		\begin{tabular}{lrr}
												& \multicolumn{1}{c}{Road network} & \multicolumn{1}{c}{Parking lots} \\ \cline{2-3} 
			Service speed           & $7.00$\SIvel{}                     & $3.33$\SIvel{}                    \\
			Deadheading speed       & $10.00$\SIvel{}                    & $5.00$\SIvel{}                    \\
			Wind speed		          & $2.00$\SIvel{}                     & $1.34$\SIvel{}                    \\
			Wind direction          & $45.00^\circ$										     &$67.50^{\circ}$                    \\
			SRLC-ILP cost           & $485$\SIs{}                           & $1,123$\SIs{}                         \\
			SRLC-ILP flight time    & $502$\SIs{}                           & $925$\SIs{}                          \\
			\BtwoAtwo{} cost        & $492$\SIs{}                           & $1,172$\SIs{}                         \\
			\BtwoAtwo{} flight time & $527$\SIs{}                           & $1,023$\SIs{}                        
		\end{tabular}
	\end{center}
	\caption{
	Operating conditions, computed coverage tour costs, and actual flight times for experiments with a quadrotor UAV.\label{tb:exp}}
\end{table}

\begin{figure}[hbtp]
	\centering
	\subfloat[Coverage tour]{%
	\includegraphics[width=0.35\textwidth]{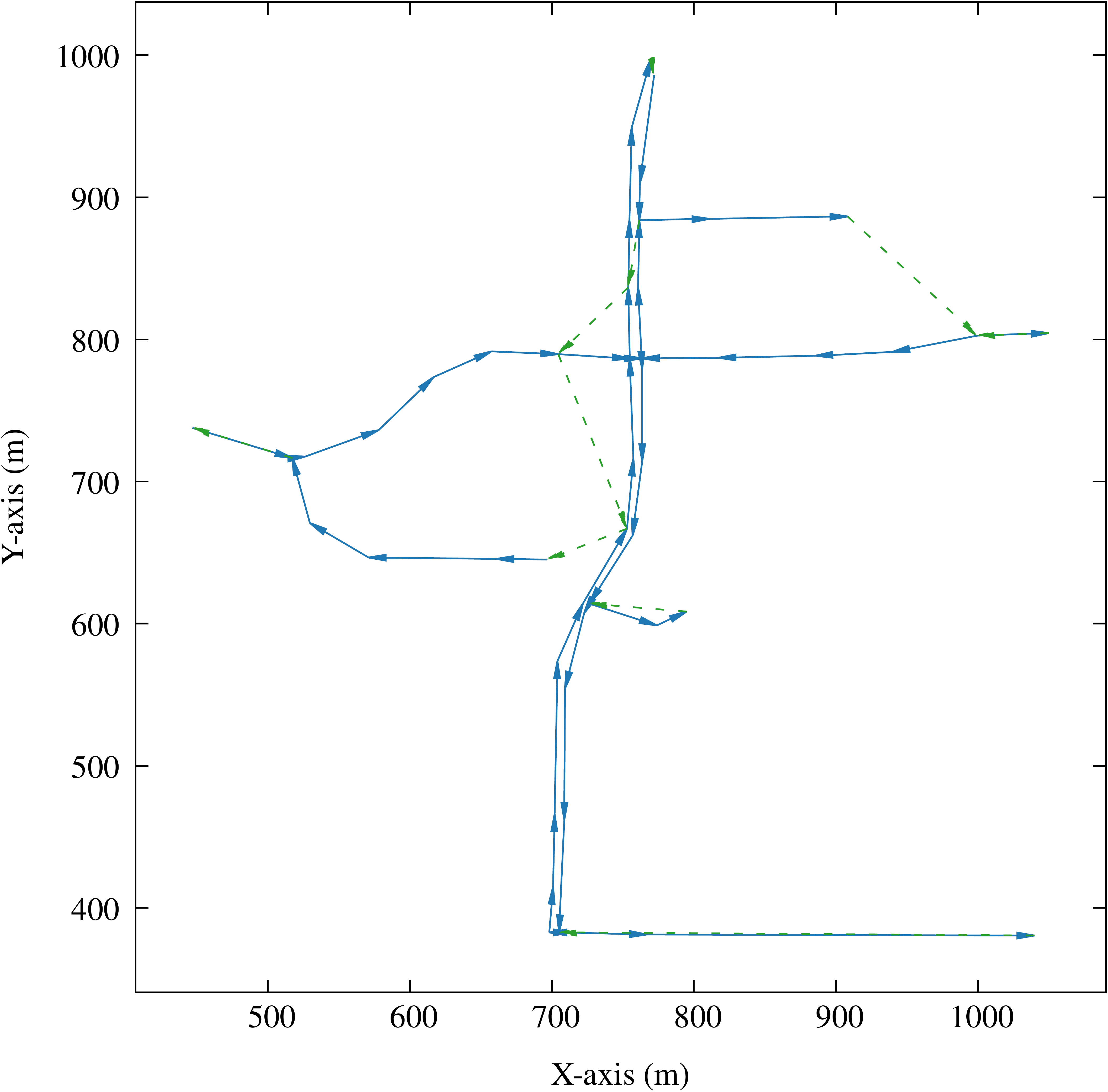}}
	\hfill
	\subfloat[Actual flight path]{%
	\includegraphics[width=0.283\textwidth]{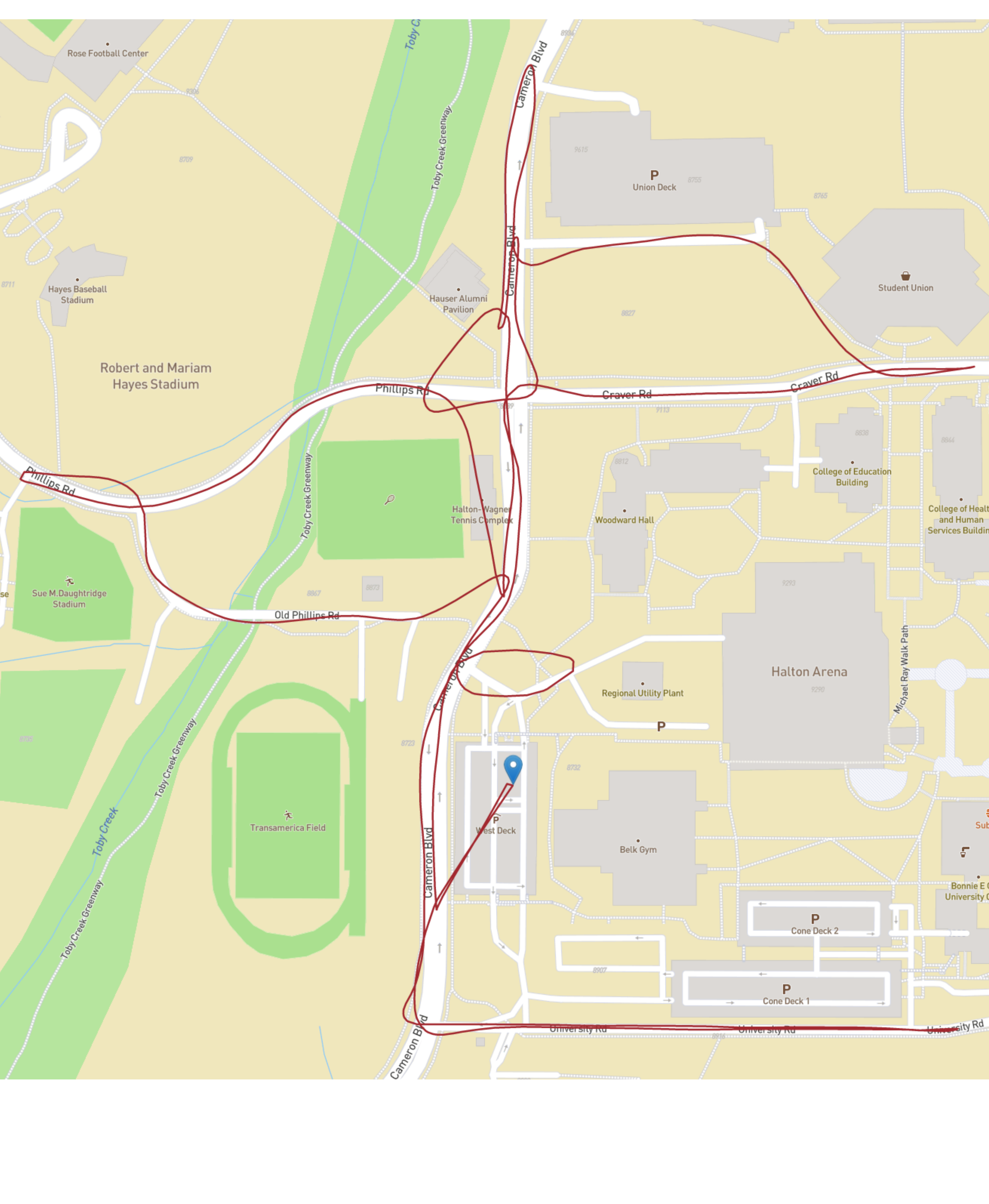}}
	\hfill
	\subfloat[Orthomosaic]{%
	\includegraphics[width=0.324\textwidth]{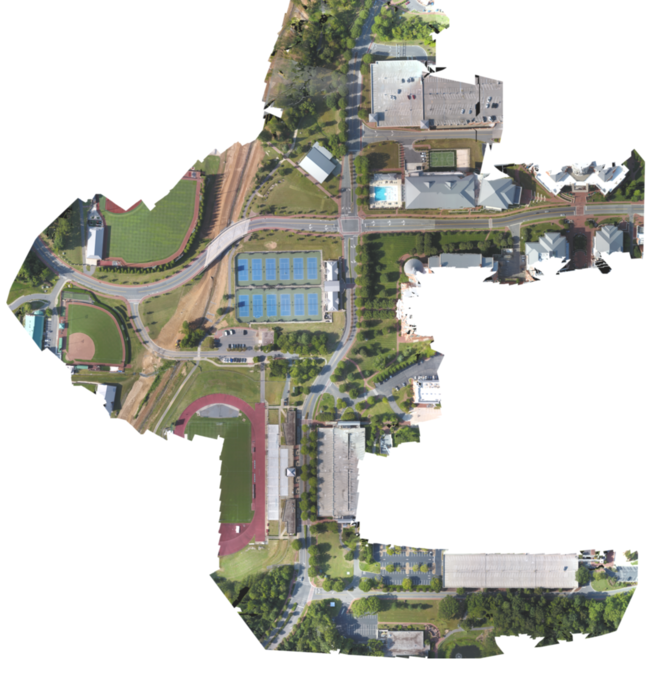}}
	\caption[Line coverage of a portion of the UNC Charlotte road network]{%
		Coverage of a portion of the UNC Charlotte road network:
		The required graph has one connected component.
		The road network has a length of 2,658\SIm{} with 48 vertices and 48 required edges.
		There are 1,128 non-required edges formed by each pair of vertices.
		(a) Coverage tour generated using the \BtwoAtwo{} algorithm.
		The cost of the solution is 492.48\SIs{}.
		The servicing travel is denoted by solid blue lines, while dashed green lines denote the deadheading travel.
		The arrowheads indicate the direction of travel.
		(b) The actual flight path of a UAV executing the coverage tour autonomously.
		The blue marker denotes the launch location of the UAV.
		(c) Orthomosaic generated from the images collected during servicing travel along the coverage tour.
		Collecting images only during servicing reduces the number of images that need to be processed for mapping and analysis.
	\label{fig:uncc}}
\end{figure}

\begin{figure}[htbp]
	\centering
	\includegraphics[width=0.5\textwidth]{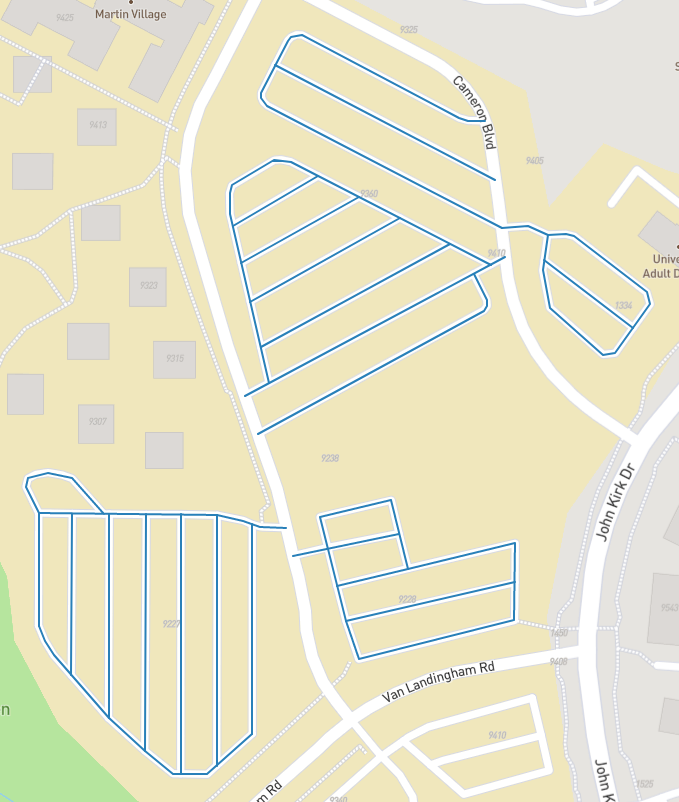}
	\caption[A network of lanes on a set of parking lots]{%
		A network of lanes specified on a set of parking lots:
		The total length of the lanes is 2,982\SIm{}.
		There are 90 vertices, 104 required edges, and 4,005 non-required edges.
		The required edges form four connected components.
	\label{fig:lot_input}}
\end{figure}
\begin{center}
	\begin{figure}[hbtp]
		\centering
		\subfloat[Coverage tour]{%
		\includegraphics[width=0.32\textwidth]{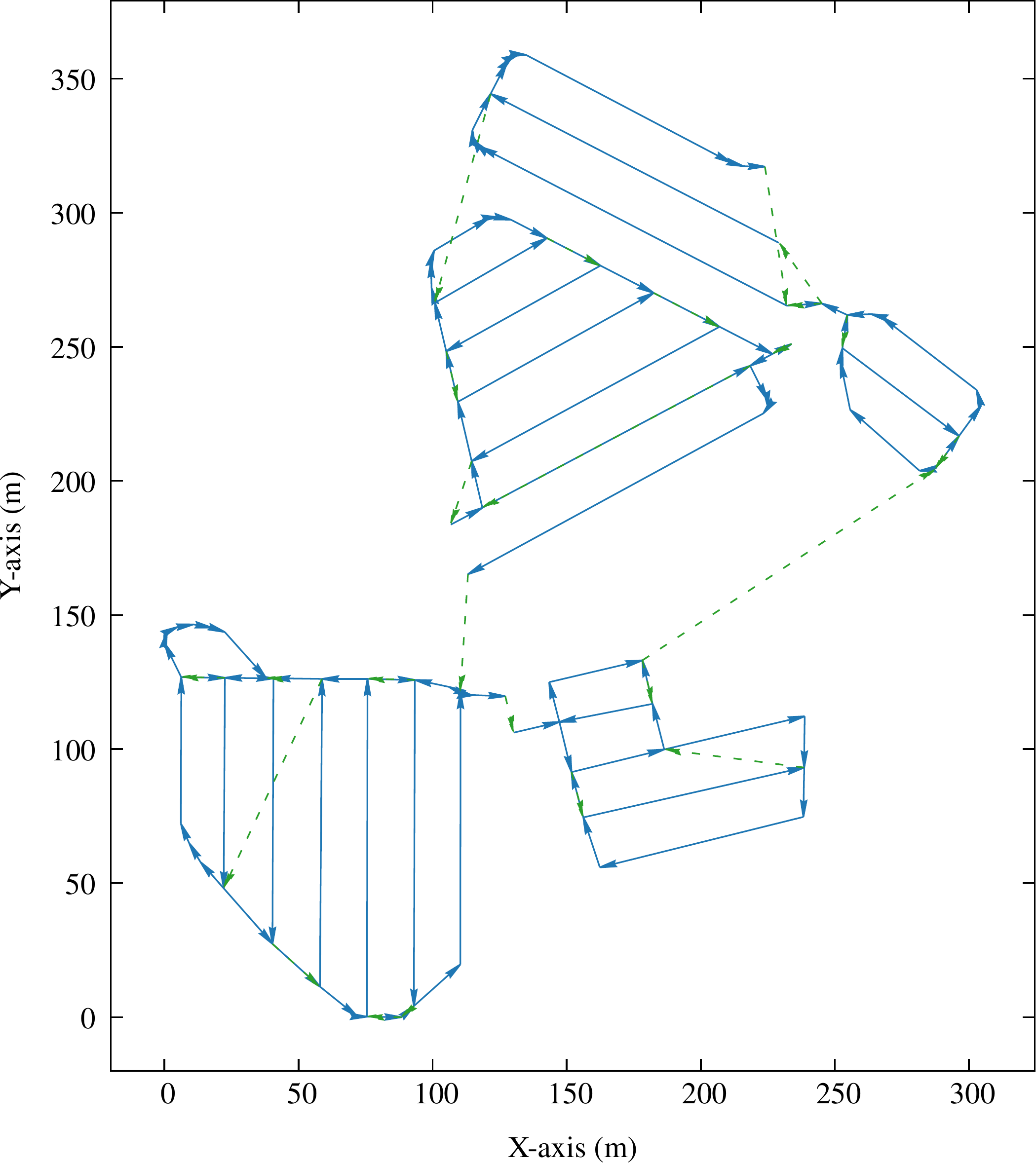}}
		\hfill
		\subfloat[Actual flight path]{%
		\includegraphics[width=0.32\textwidth]{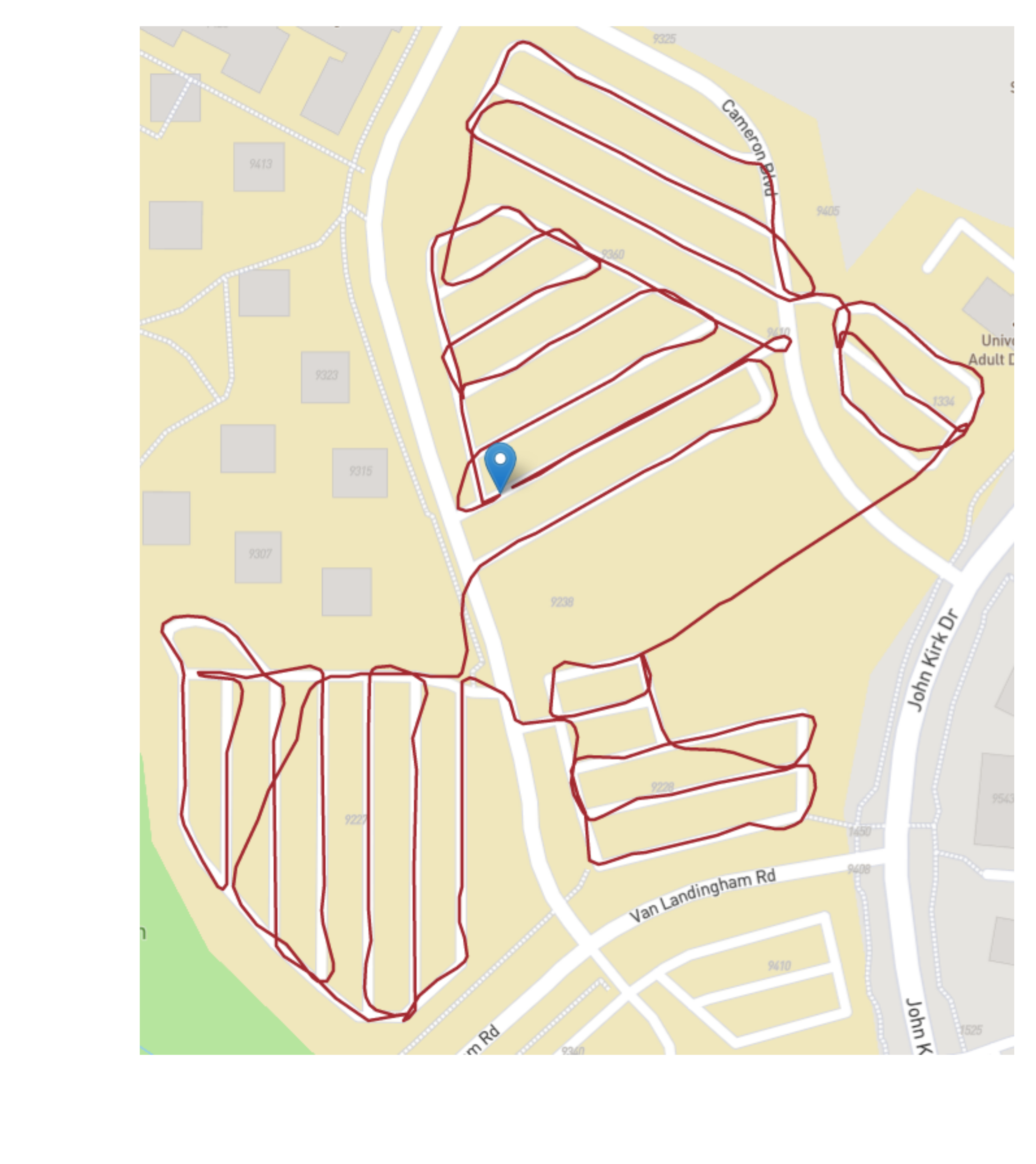}}
		\hfill
		\subfloat[Orthomosaic]{%
		\includegraphics[width=0.32\textwidth]{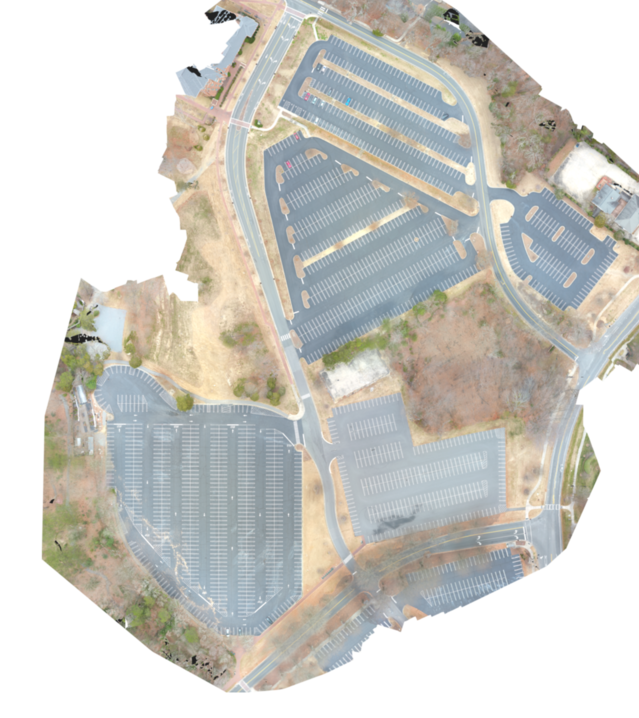}}
		\caption[Line coverage of lanes specified on a set of parking lots]{%
			Line coverage of lanes specified on a set of parking lots:
			The required graph has four connected components.
			(a) Coverage tour generated using the \BtwoAtwo{} algorithm.
			The cost of the solution is 1,172\SIs{}.
			The servicing travel is denoted by solid blue lines, while dashed green lines denote the deadheading travel.
			The arrowheads indicate the direction of travel.
			(b) The actual flight path of a UAV executing the coverage tour autonomously.
			The blue marker denotes the launch location of the UAV.
			The actual flight took 1,023\SIs{}.
			(c) Orthomosaic computed from images collected during the flight.
		\label{fig:lot}}
	\end{figure}
\end{center}

%% file: graphics/costs_b2.tex
\begin{tikzpicture}
	\small
	\begin{axis}[
		enlargelimits=true,
		xlabel = {Number of required edges},
		ylabel = {Cost difference percentage},
		legend pos=north west,
		legend cell align=left,
		legend style={font=\footnotesize},
		xmajorgrids=true,
		grid style=dashed,
		xtick={0,100,200,300,400,500,600,700,800},
		ymin=0,
		ymax=60,
		ylabel near ticks,
		width=0.8\textwidth,
		height=7cm
		]
		\addplot[
		only marks,
		color=mDarkRed,
		mark=+,
		mark size=1.7pt]
		table[x=m,y=b2a_cost, col sep=comma]
		{./graphics/results.txt};
		\addplot[
		only marks,
		color=mBlue,
		mark size=.8pt]
		table[x=m,y=b2a2_cost, col sep=comma]
		{./graphics/results.txt};
		\addplot[
		only marks,
		color=mGreen,
		mark=x,
		mark size=1.5pt]
		table[x=m,y=b3a_cost, col sep=comma]
		{./graphics/results.txt};
	\legend{\BtwoA{},\BtwoAtwo{},\BthreeA{}}
	\end{axis}
\end{tikzpicture}

%% file: graphics/time_b2.tex
\begin{tikzpicture}
	\begin{axis}[
		enlargelimits=true,
		xlabel = {Number of required edges},
		ylabel = {Computation time ($s$)},
		legend pos=north west,
		legend cell align=left,
		legend style={font=\footnotesize},
		xmajorgrids=true,
		grid style=dashed,
		xtick={0,100,200,300,400,500,600,700,800},
		ymin=0,
		ymax=3,
		ylabel near ticks,
		width=0.8\textwidth,
		height=7cm
		]
		\addplot[
		only marks,
		color=mDarkRed,
		mark=+,
		mark size=1.7pt]
		table[x=m,y=b2a_time_s, col sep=comma]
		{./graphics/results.txt};
		\addplot[
		only marks,
		color=mBlue,
		mark size=.8pt]
		table[x=m,y=b2a2_time_s, col sep=comma]
		{./graphics/results.txt};
		\addplot[
		only marks,
		color=mGreen,
		mark=x,
		mark size=1.5pt]
		table[x=m,y=b3a_time_s, col sep=comma]
		{./graphics/results.txt};
	\legend{\BtwoA{},\BtwoAtwo{},\BthreeA{}}
	\end{axis}
\end{tikzpicture}

%% file: graphics/atspvsgtsp.tex
\begin{center}
	\begin{tabular}{r}
		\small
		\begin{tikzpicture}
			\pgfplotstableread[col sep=comma]{./graphics/atspvsgtsp.csv}{\agdata};
			\begin{axis}[
				enlargelimits=true,
				xlabel = {Number of connected components},
				ylabel = {Cost difference percentage},
				legend pos=north west,
				legend cell align=left,
				legend style={font=\footnotesize},
				xmajorgrids=true,
				grid style=dashed,
				width=7.5cm,
				xtick=data,
				xticklabels from table={\agdata}{num_cc},
				ylabel near ticks
				]
				\addplot[
					only marks,
					color=mDarkRed,
					mark=+,
					mark size=1.7pt]
					table[x expr=\coordindex,y=b2a_cost]{\agdata};
				\addplot[
					only marks,
					mark=diamond*,
					color=mSteelGray,
					mark size=1.2pt]
					table[x expr=\coordindex,y=b2g_cost]{\agdata};
				\legend{\BtwoA{},\BtwoG{}}
			\end{axis}
		\end{tikzpicture}\hspace{0.5cm}
		\begin{tikzpicture}
			\pgfplotstableread[col sep=comma]{./graphics/atspvsgtsp.csv}{\agdata};
			\begin{axis}[
				enlargelimits=true,
				xlabel = {Number of connected components},
				ylabel = {Computation time (s)},
				width=7.5cm,
				legend pos=north west,
				legend cell align=left,
				legend style={font=\footnotesize},
				xmajorgrids=true,
				grid style=dashed,
				xtick=data,
				xticklabels from table={\agdata}{num_cc},
				ytick={0,2,4,6,8},
				ymin=-4,
				ylabel near ticks
				]
				\addplot[
					only marks,
					color=mDarkRed,
					mark=+,
					mark size=1.7pt]
					table[x expr=\coordindex,y=b2a_time_s]{\agdata};
				\addplot[
					only marks,
					mark=diamond*,
					color=mSteelGray,
					mark size=1.2pt]
					table[x expr=\coordindex,y=b2g_time_s]{\agdata};
				\legend{\BtwoA{},\BtwoG{}}
				\node[rotate=90,anchor=south,right,fill=white,inner sep=0] at (axis cs: 0,-4.5) {\footnotesize Surat};
				\node[rotate=90,anchor=south,right,fill=white,inner sep=0] at (axis cs: 1,-4.5) {\footnotesize Beijing};
				\node[rotate=90,anchor=south,right,fill=white,inner sep=0] at (axis cs: 2,-4.5) {\footnotesize Hong Kong};
				\node[rotate=90,anchor=south,right,fill=white,inner sep=0] at (axis cs: 3,-4.5) {\footnotesize S\~ao Paulo};
				\node[rotate=90,anchor=south,right,fill=white,inner sep=0] at (axis cs: 4,-4.5) {\footnotesize Jakarta};
				\node[rotate=90,anchor=south,right,fill=white,inner sep=0] at (axis cs: 5,-4.5) {\footnotesize Suzhou};
				\node[rotate=90,anchor=south,right,fill=white,inner sep=0] at (axis cs: 6,-4.5) {\footnotesize Wuhan};
				\node[rotate=90,anchor=south,right,fill=white,inner sep=0] at (axis cs: 7,-4.5) {\footnotesize Chennai};
				\node[rotate=90,anchor=south,right,fill=white,inner sep=0] at (axis cs: 8,-4.5) {\footnotesize Xi'an};
				\node[rotate=90,anchor=south,right,fill=white,inner sep=0] at (axis cs: 9,-4.5) {\footnotesize Nanjing};
				\node[rotate=90,anchor=south,right,fill=white,inner sep=0] at (axis cs: 10,-4.5) {\footnotesize Ahmedabad};
				\node[rotate=90,anchor=south,right,fill=white,inner sep=0] at (axis cs: 11,-4.5) {\footnotesize Guangzhou};
			\end{axis}
		\end{tikzpicture}
	\end{tabular}
\end{center}

%% file: conclusion.tex
Motivated by coverage applications for linear infrastructure such as road networks, power lines, and oil and gas pipelines, we addressed the single robot line coverage problem for autonomous aerial and ground robots.
The linear features are modeled as required edges in a graph that the robot must service.
Additional non-required edges, which do not require servicing, provide flexibility for a robot to select its path.
The two modes of travel---servicing and deadheading---permit better modeling of real-world scenarios where a robot needs to perform task-specific actions such as taking images only along specified features.
This reduces the workload of the robot, permits further optimization of the travel cost, and decreases the amount of sensor data that needs to be analyzed.
Our formulation models asymmetric cost functions and permits multiple copies of edges.
This enables one-way streets and repeated servicing of segments.

We formulated the single robot line coverage problem as an optimization problem on graphs and developed an ILP formulation that gives optimal solutions.
Formal proofs establish the correctness of the formulation.
As the problem is NP-hard, we developed approximation algorithms that provide a guarantee on the quality of the solutions.
Studying the structure of the required graph---the graph induced by the linear features---provided insights into the problem, which were used to develop the approximation algorithms.
The algorithms were developed in stages, going from a simple version of the problem to the most general one.
First, an optimal algorithm based on the minimum cost flow problem was discussed for the case where the required graph is Eulerian.
For the case where the required graph is connected but not necessarily Eulerian, a 2-approximation algorithm was developed.
Finally, an $\left(\alpha(C) + 2\right)$-approximation algorithm was given for the general case of a required graph with $C$ components, where $\alpha(C)$ is the approximation factor for an algorithm for the ATSP.
Proofs for the approximation factor were provided for each of the algorithms.
Heuristics that improve the quality of the solutions were incorporated into the algorithm, and a GTSP based alternative was evaluated.

Simulation results on a road network dataset of the $50$ most populous cities in the world show that our main algorithm computes high-quality solutions that are within $10$\% of the optimum in less than $3$\,s.
The algorithms are fast enough for rapid replanning.
Experiments with a commercial UAV were performed on a portion of the UNC Charlotte road network and on lanes of a set of parking lots to generate orthomosaic maps.
The autonomous flights resulted in fewer images that capture only the features of interest, as the images are taken only during servicing and not while deadheading.

We are currently exploring the application of our algorithms to the line coverage problem with multiple resource-constrained robots.
Our plan is to use tour-splitting techniques to generate solutions for multiple robots.
Our preliminary study indicates that very efficient solutions can be generated for the multi-robot line coverage problem using the high-quality solutions computed by our algorithms for the single robot line coverage problem.